%% file: main_ARXIV.tex
\documentclass[final]{arxiv2023} 

\input{common-header.sty}
\newcommand{\p}{\bar{p}}
\newcommand{\q}{\bar{q}}

\title[Max-Quantile Grouped Infinite-Arm Bandits]{Max-Quantile Grouped Infinite-Arm Bandits}
\usepackage{times}

\altauthor{%
 \Name{Ivan Lau} \Email{ivan.lau@rice.edu}\\
 \addr Rice University
 \AND
 \Name{Yan Hao Ling} \Email{lingyh@nus.edu.sg}\\
 \addr National University of Singapore
 \AND
 \Name{Mayank Shrivastava} \Email{mayanks4@illinois.edu}\\
 \addr University of Illinois Urbana-Champaign
 \AND 
 \Name{Jonathan Scarlett} \Email{scarlett@comp.nus.edu.sg}\\
 \addr National University of Singapore
}

\renewcommand{\cite}{\citep}

\begin{document}

\maketitle

\begin{abstract}%
    In this paper, we consider a bandit problem in which there are a number of groups each consisting of infinitely many arms.  Whenever a new arm is requested from a given group, its mean reward is drawn from an unknown reservoir distribution (different for each group), and the uncertainty in the arm's mean reward can only be reduced via subsequent pulls of the arm.  The goal is to identify the infinite-arm group whose reservoir distribution has the highest $(1-\alpha)$-quantile (e.g., median if $\alpha = \frac{1}{2}$), using as few total arm pulls as possible.  We introduce a two-step algorithm that first requests a fixed number of arms from each group and then runs a finite-arm grouped max-quantile bandit algorithm.  We characterize both the instance-dependent and worst-case regret, and provide a matching lower bound for the latter, while discussing various strengths, weaknesses, algorithmic improvements, and potential lower bounds associated with our instance-dependent upper bounds.
    
\end{abstract}

\input{intro.tex}
\input{body.tex}

\section{Conclusion}

We have introduced the problem of grouped max-quantile infinite-arm bandits, which comes with a variety of unique challenges compared to its component ``ingredients'' of infinite-arm bandits and grouped bandits.  We introduced a two-step elimination-based algorithm and discussed some evidence for the strength of its instance-dependent upper bound, as well as rigorously proving its optimality up to logarithmic factors in a worst-case sense.  For future work, we expect that it should be possible to attain an instance-dependent lower bound that attains the gap terms discussed in Appendix \ref{sec:instance_lb} while also exhibiting the improved $\epsilon$-dependence that was lacking therein.

\acks{This work was supported by the Singapore National Research Foundation (NRF) under grant number A-0008064-00-00.}

\bibliography{bibliography}

\appendix

\input{appendix_ub}

\input{appendix_bound_independent_H.tex}
\input{appendix_subroutine_finite.tex}

\input{appendix_finite}

\input{appendix_lb}

\end{document}

%% file: intro.tex
\vspace*{-2ex}
\section{Introduction}
\vspace*{-1ex}

Multi-armed bandit (MAB) algorithms are widely adopted in scenarios of decision-making under uncertainty \cite{ slivkins2019introduction,Csa18}.  In theoretical MAB studies, two particularly common performance goals are {\em regret minimization} and {\em best arm identification}, and this paper is more  related to the latter.
The most basic form of best arm identification seeks to identify a \textit{single} arm with the highest mean reward from a \textit{finite} set of arms \cite{kaufmann2016complexity}.  
Other variations include but not limited to (i) 
identifying {\em multiple arms} with highest mean rewards \cite{kalyanakrishnan2012pac};
(ii)
 identifying a single arm with a near-maximal mean reward from an \textit{infinite}  set of arms \cite{aziz2018pure,chaudhuri2018quantile};
 (iii)
 identifying, for each group of arms, a single arm with the highest mean reward 
 \cite{gabillon2011MultiBandit, bubeck2013multipleident};
 (iv)
 identifying a single arm-group from a set of groups (of arms) whose worst arm  has the highest mean reward \cite{wang2021max}.

In this paper, we consider a novel setup featuring aspects of both settings (ii) and (iv) above.
Specifically, we are given a finite number of groups, with each group having an
(uncountably) infinite number of arms, and our goal is to identify the group whose top {\em $(1-\alpha)$-quantile arm} (in terms of the mean reward) is as high as possible.  While this setup builds on existing ones, we will find that it comes with unique challenges, with non-minor differences in both the upper and lower bounds.  

A natural motivation for this problem is \emph{comparing large populations} (e.g., of users, items, or otherwise).  As a concrete example, in recommendation systems, one may wish to know which among several populations has the highest median click-through rate, while displaying as few total recommendations as possible.  Beyond any specific applications, we believe that our problem is natural to consider in the broader context of infinite-arm bandits.

Before formally introducing the problem and stating our contributions, we outline some related work.

\vspace*{-2ex}
\subsection{Related Work} \label{sec:related}
\vspace*{-1ex}

The related work on multi-armed bandits is extensive (e.g., see \cite{slivkins2019introduction, Csa18} and the references therein); we only provide a brief outline here, with an emphasis on only the most closely related works.

{\bf Finite-arm settings.} The standard (single) best arm identification problem was studied in \cite{audibert2010bestarmid,gabillon2012unified,jamieson2014best,kaufmann2016complexity,garivier2016optimal}, among others.  These works are commonly distinguished according to whether the time horizon is fixed (fixed-budget setting) or the target error probability is fixed (fixed-confidence setting), and our focus is on the latter.  In particular, we will utilize anytime confidence bounds from \cite{kaufmann2016complexity} for upper-bounding
the number of pulls.

A {\em grouped} best-arm identification problem was studied in \cite{gabillon2011MultiBandit,bubeck2013multipleident,scarlett2019overlapping}, where the arms are allocated into groups, and the goal is to find the best arm in each group.  
Another notable setting in which multiple arms are returned is that of subset selection, where one seeks to find a subset of $k$ arms attaining the highest mean rewards \cite{kalyanakrishnan2012pac,pmlr-v30-Kaufmann13,kaufmann2016complexity}.  Group structure can also be incorporated into \emph{structured bandit} frameworks \cite{huang2017structured,gupta2020unified,mukherjee2020generalized,neopane2021best}.  Perhaps the most related among these is that of max-min grouped bandits \cite{wang2021max}, which seeks a group whose worst arm is as high as possible in a finite-arm setting.   We discuss the main similarities and differences to our setting in Appendix \ref{sec:comparisons}, as well as highlighting a connection to structured best-arm identification \cite{huang2017structured}.


{\bf Infinite-arm settings.} One line of works in infinite-arm bandits assumes that the actions lie in a metric space, and the associated mean rewards satisfy some smoothness condition such as being locally Lipschitz \cite{kleinberg2008multi, bubeck2008online, bubeck2011x, grill2015black, kleinberg2019bandits}.   
More relevant to our paper is the line of works considering a \emph{reservoir distribution} on arms; throughout the learning process, new arms can be requested and previously-requested ones can be pulled.  Earlier works in this direction focused primarily on regret minimization (e.g., \cite{berry1997bandit, wang2008infinitely, bonald2013two, david2014infinitely,li2017infinitely,kalvit2021bandits}), and several recent works have considered pure exploration (e.g., \cite{carpentier2015simple,jamieson2016power,chaudhuri2018quantile,aziz2018pure,chaudhuri2019pac,ren2019exploring,katz2020true,zhang2021quantile}).  Throughout the paper, we particularly focus on the work \cite{aziz2018pure}, which studies the problem of finding a single arm in the top $(1-\alpha)$-quantile.  See Appendix~\ref{sec:comparisons} for a discussion of some key similarities and differences.

The idea of measuring the performance against a quantile has appeared in several of these works, including \cite{chaudhuri2018quantile,aziz2018pure}.  Related notions include finding a ``good'' arm in a scenario where there are two types of arm \cite{jamieson2016power}, and finding a subset of a top fraction of arms in a finite-arm setting \cite{chaudhuri2019pac,ren2019exploring}.  The notion of finding an optimal quantile has also appeared extensively in risk-aware bandits \cite{tan2022survey}, but these problems consist of a single set of arms and consider quantiles of the reward distributions, which is handled very differently to quantiles of reservoir distributions.



%% file: body.tex

\vspace*{-2ex}
\section{Problem Setup and Contributions} 
\label{sec:setup}
\vspace*{-1ex}

\textbf{Arms and rewards.}
We first describe the problem aspects that are the same as regular MAB problems. 
We are given a collection $\gA$ of  arms, which has some unknown mean-reward mapping $\mu \colon \gA \to \R$.  
In each round, indexed by $t \ge 1$, the algorithm 
pulls one or more arms\footnote{We will find it useful to let $t$ index ``rounds'' each possibly consisting of several arm pulls, but we will still be interested in the \emph{total} number of arm pulls.} in $\gA$ and observes
their corresponding rewards.
We consider the stochastic reward setting, in which for each arm $j \in {\cal A}$, the observations of its
reward $\{X_{j,t}\}_{t \ge 1}$ are i.i.d. random variables 
from some distribution with mean $\mu_j \coloneqq \mu(j)$.  
It is also useful to define the empirical mean of arm $j$ at round $t$, defined as 
\begin{equation}
	\label{eq:empirical_mean}
	\hat{\mu}_{j, T_j(t)} \coloneqq \frac{1}{T_j(t)}\sum_{\substack{\tau \in \{1,\dotsc,t\} \,:\, \\ \text{arm $j$ pulled}}}X_{j,\tau},
\end{equation}
where $T_{j}(t) \le t$ is the number of pulls of arm $j$ up to round $t$.  In standard best-arm identification problems, the goal is to identify the best arm with high probability using as few pulls as possible.

We will restrict our attention to classes of arm distributions that are \emph{uniquely parametrized by their mean}, e.g., Bernoulli($\mu$) or $N(\mu,\sigma^2)$ with $\sigma^2 > 0$ being the same for every arm.  By doing so, the notion of a reservoir distribution (see below) can be introduced using probability distributions on $\RR$, which is significantly more convenient compared to more general distributions.

\textbf{Infinite-arm and group notions.}
In our setup, the number of arms $\gA$ is 
(potentially uncountably) infinite.
Furthermore, these arms are partitioned into a finite number of \textit{disjoint} groups, with 
each group having a (potentially uncountably) infinite number of arms. The set of these disjoint groups is denoted by $\gG$.
The mean rewards for each group $G\in \gG$ form a probability space 
$\left(
\ExpRewards_{G}, \ExpRewardEvents_{G}, \ExpRewardMeasure_G \right)$,
where 
$\ExpRewards_G = \mu(G) =
\left\{\mu_j \colon j \in G \right\}$.
For each group~$G$, we define a corresponding \emph{reservoir distribution} CDF $\ExpRewardCDF_G \colon  \ExpRewards_{G} \to [0, 1]$, as well as a quantile function
$\ExpRewardInvCDF_G \colon [0, 1] \to  \ExpRewards_{G}$,
by
\begin{equation}
	\label{eq:CDF}
	\ExpRewardCDF_G (\tau)  \coloneqq
	\ExpRewardMeasure_G ( \Set{\mu \le \tau} )
	\quad \text{ and } \quad 
	\ExpRewardInvCDF_G(p) := \inf \Set{ \expReward : F_G(\expReward) \ge p },
\end{equation}
where $F_G$ is assumed to have a bounded support. 
Our goal is to identify a group in $\gG$ with the highest $(1-\alpha)$-quantile, i.e., a group $G$ satisfying
\begin{equation}
	\label{eq:exact_best_group}
	\ExpRewardInvCDF_{G}(1 - \alpha) = \max_{G' \in {\cal G}} \ExpRewardInvCDF_{G'}(1 - \alpha).
\end{equation}
Observe that when $\alpha$ is close to zero, this problem resembles that of finding a near-maximal arm across all groups, whereas when $\alpha$ is close to one, the problem resembles the max-min problem of finding the group whose worst arm is as high as possible \cite{wang2021max}.  Throughout the paper, we treat $\alpha \in (0,1)$ as a fixed constant (e.g., $\alpha = \frac{1}{2}$ for the median), meaning its dependence may be omitted in $O(\cdot)$ notation.

Throughout the course of the algorithm, the following can be performed:
\begin{itemize}[itemsep=0ex,topsep=0pt]
	\item The algorithm can request to receive one or more additional arms from any group of its choice; if that group is $G \in \gG$, then the arm mean is drawn according to $F_G$.
	\item Among all the arms requested so far, the algorithm can perform \textit{pulls} of the arms and observe the corresponding rewards, as usual.
\end{itemize}
We are interested in keeping the total number of arm pulls low; the total number of arms requested is not of direct importance (similar to previous infinite-arm problems such as \cite{aziz2018pure}).

It will be convenient to uniquely index each arm in a given group $G$ by $j \in [0,1]$ such that any new arm drawn from the reservoir distribution has its $j$ value drawn uniformly from $[0,1]$, and has mean reward $\mu_{G,j} = F^{-1}_G(j)$.\footnote{We may assume that the same $j$ value is never drawn twice, since this is a zero-probability event.  We emphasize that with $\mu = F^{-1}_G(j)$, the assumption of unique $j$ values does not necessarily imply unique arm means, as we still allow $F_G$ to have mass points.}
We will use this convention, but it is important to note that whenever the algorithm requests an arm, its index $j$ remains unknown.  We also note that two different indices $j \ne j'$ in a given group could still have the same means (i.e., $\mu_{G,j} = \mu_{G,j'}$), e.g., when the underlying reservoir distribution is discrete.


\textbf{$\epsilon, \Delta$-relaxations.}
When no assumptions are made on the reservoir distributions, one may encounter scenarios where the $(1-\alpha)$-quantile of a given group is arbitrarily hard to pinpoint (e.g., because the underlying CDF has a near-horizontal region, implying a very low probability of any given arm being near the precise quantile). 
To alleviate this challenge, we add an $\epsilon$-relaxation on $\alpha$
for some $\epsilon < \min(\alpha, 1 - \alpha)$
to limit the effort spent on identifying a quantile that is hard to pinpoint.  
In particular, we relax the task 
into finding a group $G$ satisfying
$
\ExpRewardInvCDF_{G}(1 - \alpha + \epsilon)
\ge \max_{G' \in {\cal G}} \ExpRewardInvCDF_{G'}(1 - \alpha - \epsilon).
$
Moreover, we further relax the task by only 
requiring the chosen group $G$ satisfies
\begin{equation}
	\label{eq:best_group_relax_eps_Delta}
	\ExpRewardInvCDF_{G}(1 - \alpha + \epsilon)
	\ge 
	\max_{G' \in {\cal G}} \ExpRewardInvCDF_{G'}(1 - \alpha - \epsilon)  - \Delta,
\end{equation}
for some $\Delta > 0$.
This relaxation allows us to limit the effort on distinguishing arms
(potentially from different groups) whose expected rewards are very close to each other; analogous relaxations are common in standard best-arm identification problems.
The general goal of the algorithm for our setup is to identify a group satisfying \eqref{eq:best_group_relax_eps_Delta} with high probability while using as few arm pulls as possible.

\textbf{Summary of contributions} With the problem setup now in place, we can now summarize our main contributions (beyond the problem formulation itself):
\begin{itemize}[itemsep=0ex,topsep=0pt,parsep=0pt]
	\item We introduce a two-step algorithm based on first requesting a fixed number of arms for each group, and then running a finite-arm subroutine.  Our main result for this algorithm is an instance-dependent upper bound on the number of pulls to guarantee \eqref{eq:best_group_relax_eps_Delta} with high probability (Corollary \ref{cor:samplebound_independent_H}), expressed in terms of the reservoir distributions and fundamental gap quantities introduced in Section \ref{sec:num_pulls}.
	\item In addition, we establish a worst-case upper bound in terms of $\epsilon$ and $\Delta$ alone (Eq.~\eqref{eq:weakened}), and a guarantee for the finite-arm setting (Theorem \ref{thm:samplebound}).  We believe that these should be of independent interest.
	\item We show that adapting our two-step algorithm to a multi-step algorithm can lead to a better instance-dependent bound (Corollary \ref{cor:samplebound_improved}), and discuss how the resulting gap terms may be similar to those of a potential instance-dependent lower bound (Appendix~\ref{sec:compare_aziz}), though formalizing the latter is left for future work.
	\item By deriving a worst-case lower bound (Theorem \ref{thm:lb}) and comparing it with our upper bound, we show that for worst-case instances, the optimal number of arm pulls scales as $\frac{|\mathcal{G}|}{\Delta^2\epsilon^2}$ (for $|\mathcal{G}| \ge 2$) up to logarithmic factors.  
\end{itemize}
Some of the innovations in our analysis include (i) suitably identifying the relevant ``gaps'' in the finite-arm elimination algorithm and adapting the analysis accordingly; (ii) setting up the relaxation~\eqref{eq:best_group_relax_eps_Delta} and determining how this impacts the number of arms to request and how to characterize the high-probability behavior of their gaps; (iii) suitably extending the two-step algorithm to a multi-step algorithm; and (iv) proving the lower bound via a likelihood-ratio based analysis that is distinct from others in the bandit literature to the best of our knowledge.

\vspace*{-2ex}
\section{Algorithm and Upper Bound}
\vspace*{-1ex}
\input{2-step_algorithm}


\vspace*{-2ex}
\section{Lower Bounds}
\vspace*{-1ex}
\input{lower_bounds}

%% file: 2-step_algorithm.tex

\label{sec-fc-upper}

In this section, we introduce our main algorithm and provide its performance guarantee.  We make the following standard assumptions on the reward distributions.
\begin{assumption} \label{as:noise}
    For every arm $j$ (in every group $G$; the group dependence is left implicit here), we assume that the mean reward  $\mu_j$ is bounded in $[0,1]$,\footnote{Any finite interval can be shifted and scaled to this range.} and that the reward distribution is sub-Gaussian with parameter $\sigma^2 \le 1$.\footnote{The upper bound $\sigma^2 \le 1$ is for convenience in applying known confidence bounds, and can easily be relaxed.}
    That is,  for each arm $j$, and 
    for each $\lambda \in \mathbb{R}$,
    if $X$ is a random variable drawn from the arm's reward distribution, then $\mathbb{E}[X] = \mu_j$ and $\mathbb{E}[e^{\lambda (X - \mu_{j})}] \leq \exp(\lambda^2 \sigma^2/2)$.  Moreover, we assume that all of the reward distributions come from a common family of distributions that are uniquely parametrized by their mean (e.g., Bernoulli, or Gaussian with a fixed variance).
\end{assumption}
\vspace*{-2ex}
\subsection{Description of the Algorithm}
\label{kl-lucb}
\vspace*{-1ex}

We present our two-step algorithm in Algorithm~\ref{alg:main}.  This algorithm uses a finite-arm best-quantile identification (BQID) algorithm $\texttt{FiniteArmBQID}$ as a sub-routine, and its description is deferred to Algorithm~\ref{alg:elimination}  
in Appendix~\ref{sec:se}.  For now, we only need to treat this step in a ``black-box'' manner with certain guarantees outlined at the end of this subsection.

Our algorithm takes a two-step approach of first requesting a fixed number of arms from each group, and then running a finite-arm algorithm.  This approach was taken in \cite{aziz2018pure} for the (non-grouped) problem of finding an arm within a given quantile, but the results and analysis turn out to be quite different; see Appendix~\ref{sec:compare_aziz} for further discussion.  In Section \ref{sec:improve}, we will discuss further improvements via a multi-step approach.

\begin{algorithm}
    \caption{Main Algorithm}
    \label{alg:main}
    \begin{algorithmic}[1]
        \Require~Finite set of infinite arm groups $\cal G$, parameters 
        $\alpha, \epsilon, \Delta, \delta \in (0,1)$
        where
        $\delta < \epsilon < 
        \min(\alpha, 1- \alpha)$
        
        \State $N \coloneq N(\epsilon, \delta) =
        \big \lceil
        \frac{1}{2\epsilon^2} \log \frac{2|\gG|}{\delta}\big\rceil$
        \label{line:N}
        
        \For {each $G \in \gG$} 
        \label{line:forloop}
        \State 
        Request $N$ arms from $G$ according to
        $F_G$ to form  a (random) finite arm group $H$
        \label{line:generate_H}
        \EndFor 
        
       \State Let $\gH = \left\{H_1, H_2, \dots, H_{|\gG|} \right\}$ be a finite set of finite arm groups
       
       \State  Run \texttt{FiniteArmBQID}
       (Algorithm~\ref{alg:elimination} in {Appendix}~\ref{sec:se})
       with input $(\gH, \alpha, \Delta, \delta)$ to identify a group
       $\widehat{H} \in \gH$ 
       \label{line:finite_alg}
       
       \State Return the group $\widehat{G}$ corresponding to $\widehat{H}$.
       \label{line:output_mainalg}
    \end{algorithmic}
\end{algorithm}

First, for each group $G \in \gG$,
we randomly and independently request $N = \big \lceil
        \frac{1}{2\epsilon^2} \log \frac{2|\gG|}{\delta} 
        \big\rceil$ arms from
        the reservoir distribution $F_G$ to form a (random) arm
group $H$ of size $N$.
 This yields a set $\gH = \left\{H_1, H_2, \dots, H_{|\gG|} \right\}$ of  finite-arm groups, each with means denoted by $\{\mu_{H,j}\}_{j \in H}$ with continuous indexing $j \in [0,1]$ as discussed in Section \ref{sec:setup}.  We make use of the corresponding (discrete) probability spaces 
 $\left(\ExpRewards_{H}, \ExpRewardEvents_{H}, \ExpRewardMeasure_H \right)$
of mean rewards, with CDFs  and quantile functions given by
\begin{equation}
    \ExpRewardCDF_H (\tau)  \coloneqq
		\ExpRewardMeasure_H ( \Set{\mu \le \tau} ) = \frac{1}{|H|} \sum_{j \in H} \boldsymbol{1}\{ \mu_{H,j} \le \tau \}
  \quad \text{ and } \quad
  \ExpRewardInvCDF_H(p) := \inf \Set{ \expReward : F_H(\expReward) \ge p }.
  \label{eq:forwardH}
\end{equation}
 In Line \ref{line:finite_alg} of Algorithm \ref{alg:main}, our requirement on the finite-arm subroutine \texttt{FiniteArmBQID} is that it identifies a group $\widehat{H} \in \gH$ with a $(1-\alpha$)-quantile that is at most $\Delta$-suboptimal:
\begin{equation}
\label{eq:optimal H}
    \ExpRewardInvCDF_{\widehat{H}}(1 - \alpha)
     \ge
    \max_{H \in \gH} \ExpRewardInvCDF_{H}(1 - \alpha)   - \Delta.
\end{equation}
We then simply return the infinite-arm group $\widehat{G}$ corresponding to $\widehat{H}$ (the two have a trivial one-to-one mapping).  The high-level ideas behind the (elimination-based) algorithm we use for \texttt{FiniteArmBQID} will be introduced in Section \ref{sec:num_pulls}, and the full details will be given in 
Appendix~\ref{sec:subroutine_details}.

\vspace*{-2ex}
\subsection{Correctness} \label{sec:crrect_main}
\vspace*{-1ex}

Our theoretical analysis is done via a series of intermediate results, and we defer the proof details to Appendix \ref{sec-correctness}.  We will frequently make use of two high-probability events, defined as follows:
\begin{itemize}[itemsep=0ex, topsep=0pt,parsep=0pt]
    \item \underline{Event $A$}: 
    For each $G \in \gG$ and the corresponding finite group $H$ (generated in Line~\ref{line:generate_H} of Algorithm~\ref{alg:main}),  we have
    \begin{equation}
        \label{eq:sampled_quantile_sandwiched}
        \ExpRewardInvCDF_{G}(1 - \alpha - \epsilon)
        \le \ExpRewardInvCDF_{H}(1 - \alpha)
        \le \ExpRewardInvCDF_{G}(1 - \alpha + \epsilon).
    \end{equation}
    That is, the mean reward of the $(1-\alpha)$-quantiles of the sampled arms $H$   
    is between the mean rewards of the 
    $(1-\alpha \pm \epsilon)$-quantile of the arms in $G$.
    \item \underline{Event $B$}:
    The group $\widehat{H}$ returned by 
    Algorithm~\ref{alg:elimination}, 
    i.e., Line~\ref{line:finite_alg} of Algorithm~\ref{alg:main}
    satisfies \eqref{eq:optimal H}.
\end{itemize}
In Appendix \ref{sec-correctness}, we show that these each hold with probability at least $1-\delta$, and then deduce the following correctness guarantee.
\begin{theorem}
    \label{thm:correctness}
    Consider Algorithm 1 with inputs $(\gG, \alpha, \epsilon, \Delta, \delta)$ as defined in Section~\ref{sec:setup}.
    Under Assumption~\ref{as:noise}, with probability at least $1- 2\delta$,  Algorithm~\ref{alg:main} identifies a group $G \in \gG$
    satisfying \eqref{eq:best_group_relax_eps_Delta}.
\end{theorem}

Here and subsequently, the fact that \texttt{FiniteArmBQID} is performed via Algorithm \ref{alg:elimination} (Appendix~\ref{sec:se}) is left implicit in most of our formal statements.


%

\vspace*{-2ex}
\subsection{Number of Arm Pulls in Finite-Arm Subroutine} \label{sec:num_pulls}
\vspace*{-1ex}

The arm pulls in Algorithm~\ref{alg:main} occur entirely within the \texttt{FiniteArmBQID} subroutine (Algorithm~\ref{alg:elimination} in {Appendix}~\ref{sec:se}). 
This algorithm and its analysis will be generally similar to \cite[Algorithm 1]{wang2021max}, though with different details.
Their algorithm can be viewed as taking $\alpha = 1$ and $\Delta = 0$ in our notation, whereas we are interested in an \textit{arbitrary} $\alpha \in (0,1)$ and we allow for $\Delta$ to be positive.  Moreover, in \cite{wang2021max} the groups are allowed to overlap; we avoided such considerations because they seem tricky to formulate neatly in the infinite-arm setting, but in the finite-arm setting of this section it would be straightforward to incorporate.
However, we emphasize that the key distinction is not these details, but rather the challenges of
incorporating the finite-arm algorithm into an infinite-arm framework.

A complete description of the finite-arm subroutine is given in Algorithm~\ref{alg:elimination} in Appendix~\ref{sec:subroutine_details}, and its analysis is given in Appendix~\ref{sec:appendix_finite_alg}. 
While the details are deferred to later, we are in a position to proceed here with a ``black-box'' statement of its number of arm pulls, along with some intuition on why certain instance-dependent quantities arise.
Algorithm~\ref{alg:elimination}
takes 
a finite set $\gH$ of finite arm-groups and parameters $\alpha, \Delta, \delta \in (0,1)$ as input,
and seeks to output a group $H \in \gH$ satisfying \eqref{eq:optimal H}
with probability at least $1-\delta$. 
It is based on successive elimination\footnote{We expect that other approaches such as LUCB-type algorithms could also provide similar guarantees.  However, see Appendix \ref{sec:structured} for a discussion on how an existing LUCB-based bound for a general structured bandit framework can be weaker than ours.}, and as with previous algorithms of this kind, it stops pulling a given arm altogether once it is ``no longer of interest''.  
To formalize this, each arm $j \in H$ is assigned a gap value $\Delta_{H,j}$, and we will show that once the empirical means and true means are at most $\Delta_{H,j}/4$ apart, arm $j \in H$ can stop being pulled.  The intuition behind why arms can be eliminated is as follows:
\begin{itemize}[itemsep=0ex,topsep=0pt,parsep=0pt]
    \item It may occur that $\mu_{H,j}$ is already known accurately enough that any further accuracy provides no additional information about the group's $(1-\alpha)$-quantile.
    \item It may be that $H$ is known to be a suboptimal group, so \emph{all} its arms can stop being pulled.
    \item It may be that the optimal group is found and the algorithm can terminate.
    \item It may be that all remaining groups are known to satisfy \eqref{eq:optimal H} and the algorithm can terminate.
\end{itemize}
These four cases will be associated with gaps denoted by $\Delta'_{H,j}$, $\Delta_H$, $\Delta_0$, and $\Delta$ respectively (to be defined shortly), leading to the following overall gap for arm $j \in H$:
\begin{equation}  
\label{def:gap}
    \Delta_{H,j} \coloneqq \max\left\{\Delta, \Delta_H, \Delta_0, \Delta_{H,j}' \right\}  \in [\Delta, 1]. 
\end{equation}
We now proceed formally.
Let $H^*$ be a group with a highest $(1-\alpha)$-quantile arm,\footnote{Due to the randomness in generating the finite-arm groups, $H^*$ does not necessarily coincide with the optimal infinite-arm group.} i.e.,
$H^* \in \argmax_{H \in \gH} \ExpRewardInvCDF_{H}(1 - \alpha)$.
If multiple groups satisfy this condition, then $H^*$ denotes an arbitrary single one of them.  For each group $H$, we define the reward-gap $\Delta_H$ of group $H$ as the
difference between the $(1-\alpha)$-quantile arm of  $H^*$ and the $(1-\alpha)$-quantile arm of group $H$: 
\begin{equation}
\label{eq:Delta_H}
    \Delta_H \coloneqq \ExpRewardInvCDF_{H^*}(1 - \alpha) 
    -  \ExpRewardInvCDF_{H}(1 - \alpha).
\end{equation}
Similarly, $\Delta_0$ indicates the difference between the $(1-\alpha)$-quantile arms of group $H^*$ and the remaining groups $H \in (\gH \setminus H^*)$, i.e.,
\begin{equation}
\label{eq:Delta_0}
    \Delta_0 \coloneqq 
    \min_{H \in \gH, H \ne H^*} 
     \Delta_H =
      \ExpRewardInvCDF_{H^*}(1 - \alpha) 
    - \max_{H \in \gH, H \ne H^*} \ExpRewardInvCDF_{H}(1 - \alpha).
\end{equation}
Note that $\Delta_0 > 0$ if and only if there is a unique optimal group.  
Finally,  $\Delta_{H,j}'$ is the difference between the mean reward of arm $j$ and the $(1-\alpha)$-quantile arm of the group it is in, i.e.,
    \begin{equation}
    \label{eq:delta_prime}
        \Delta_{H,j}' \coloneqq 
          \left| \mu_{H,j} - \ExpRewardInvCDF_{H}(1 - \alpha) \right|.
    \end{equation}
The remaining term $\Delta$ in \eqref{def:gap} is already specified as part of the problem, and is included to relax the success criterion for very challenging instances where $\max\{\Delta_H, \Delta_0, \Delta_{H,j}'\} \approx 0$ or even $\max\{\Delta_H, \Delta_0, \Delta_{H,j}'\} = 0$.
Having defined $\Delta_{H,j}$, we now state
an upper bound on the total number of arm pulls by \texttt{FiniteArmBQID}.

\begin{theorem}
\label{thm:samplebound}
Let $(\gG, \alpha, \epsilon, \Delta, \delta)$ be a valid input of Algorithm~\ref{alg:main}
and $\gH$ be the set of random groups formed
in Lines~\ref{line:forloop}-\ref{line:generate_H} of Algorithm~\ref{alg:main}.  There exists a choice of \texttt{FiniteArmBQID} (see Algorithm~\ref{alg:elimination} in Appendix~\ref{sec:se}) yielding the following:
Under Assumption~\ref{as:noise}, with probability at least $1- \delta$, the algorithm identifies a group $H \in \gH$ satisfying \eqref{eq:optimal H} and uses a total number of arm pulls upper bounded by
    \begin{align}
    \label{eq:armpullboundrandom}
        T(\epsilon, \delta, \Delta) 
        &\leq 
        \sum\limits_{H \in \gH}
        \sum_{j=1}^N
        \frac{c}{\Delta_{H, j}^2}
       \log 
        \left(
            \frac{|\gH| N }{\delta}
            \log \frac{1}{\Delta_{H, j}^2}
        \right) \\
     & \le d \left(  
     \frac{|\gG|}{\epsilon^2 \Delta^2}  \right)
     \left(
     \log^2 \frac{|\gG|}{\delta}
     + 
     \left( \log \frac{|\gG|}{\delta} \right)
     \left( \log \log \frac{1}{\Delta} \right)
     \right)       \label{eq:weakened}
    \end{align}
where $c$ and $d$ are universal constants, the gaps $\Delta_{H, j}$ as defined in \eqref{def:gap} are random variables, and
$N =
\big\lceil
\frac{1}{2\epsilon^2} \log \frac{2|\gG|}{\delta} \big\rceil$
is as given in Line~\ref{line:N}
of Algorithm~\ref{alg:main}.
\end{theorem}

\begin{proof}
  The first line follows from Theorem~\ref{thm:ub_se} in Appendix~\ref{sec:se},
  whose proof is given in Appendix~\ref{sec:appendix_finite_alg}.  The second line lower bounds each gap by $\Delta$, and performs asymptotic simplifications along with $|\gG|=|\gH|$ and the assumption $\delta < \epsilon$ specified in Algorithm \ref{alg:main}.
\end{proof}
%
\vspace*{-2ex}
\subsection{Number of Arm Pulls Without $\gH$ Dependence} \label{sec:num_pulls_without}
\vspace*{-1ex}
Since each finite group $H \in \gH$ of arms
consists of arms randomly sampled from their corresponding~$G$, the gaps 
$\Delta_H, \Delta_0, \Delta'_{H, j}$
 are random variables, and
 so $\Delta_{H, j} = \max\{ \Delta, \Delta_H, \Delta_0, 
\Delta'_{H, j } \}$ are also random variables.  Our next step is to bound these gaps with high probability.

For $\Delta_H$ and $\Delta_0$, Event A above readily yields the following lower bounds:
\begin{gather}
    \Delta_H \ge
    \widetilde{\Delta}_G \coloneqq
    \max_{G' \in \gG}  
    \ExpRewardInvCDF_{G'}(1 - \alpha - \epsilon)
    -  \ExpRewardInvCDF_{G}(1 - \alpha + \epsilon) \label{eq:Delta_H_lower_bound} \\
    \Delta_0 \ge
    \widetilde{\Delta}_0 
    \coloneqq \max\limits_{G \in \gG} \ExpRewardInvCDF_{G}(1 - \alpha - \epsilon)
    - \max\limits_{G'\in \gG, \ G' \ne G_{\epsilon}^*} \ExpRewardInvCDF_{G'}(1 - \alpha + \epsilon),
    \label{eq:Delta_0_lower_bound}
\end{gather}
where $G_{\epsilon}^* \in \argmax_{G \in\gG} \ExpRewardInvCDF_{G}(1 - \alpha + \epsilon)$. See Appendix \ref{app:DeltaH0} for the details.

In contrast, based on the high-probability events established so far, there is no non-trivial lower bound on $\Delta'_{H, j}$.  For example, if the mean rewards of all arms sampled for group $H$  are concentrated around the $(1-\alpha)$-quantile, then $\Delta'_{H, j}$ will be small for each $j$.  We address this difficulty by introducing further high-probability events under which $\Delta'_{H, j}$ can be suitably lower bounded.

Let $\gH = \{H\}$ be the set of groups formed in Lines~\ref{line:forloop}-\ref{line:generate_H} of Algorithm~\ref{alg:main}.
We partition the arms in $H$ as follows. Let 
$m$ be the smallest integer such that 
$ (1-\alpha) - 
    \left\lfloor
    \frac{1-\alpha}{\epsilon}
    \right\rfloor \epsilon 
    + m \epsilon \ge 1$, i.e.,
\begin{equation}
\label{eq:number_of_buckets}
     m \coloneqq
     \min \left\{
     k \in \mathbb{N}: k \ge
     \frac{\alpha}{\epsilon} + \left\lfloor \frac{1-\alpha}{\epsilon} \right\rfloor
     \right\}
     \in \bigg\{ \left\lfloor \frac{1}{\epsilon} \right\rfloor, \left\lceil \frac{1}{\epsilon} \right\rceil \bigg\} \ge 3,
\end{equation}
where the lower bound of $3$ follows since $\epsilon < \min(\alpha,1-\alpha) \le \frac{1}{2}$ (Line~1 of Algorithm \ref{alg:main}). 
We partition $H$ into $m+1$ disjoint subsets:
\begin{equation}
\label{def:S_H_i}
    S_{H, i}
    \coloneqq
    \begin{cases}
    H \cap [0, b_{1}) =
    \{
    j \in H \mid
    0
        \le j
        <
        b_1
    \}
    &
    \quad \text{for } i =0, \\
        H \cap [b_i, b_{i+1}) =
    \{
    j \in H \mid
    b_{i}
        \le j
        <
        b_{i+1} = b_i + \epsilon
    \}
    &
    \quad \text{for } 1 \le i \le m-1, \\
    H \cap [b_m, 1] =
    \{
    j \in H \mid
    b_{m}
        \le j
        \le
        1
    \}
    &
    \quad \text{for } i = m,
    \end{cases}
\end{equation}
where 
\begin{equation}
\label{eq: b_i}
    b_{i} \coloneqq
        \left(1-\alpha\right) - 
        \left\lfloor
        \frac{1-\alpha}{\epsilon}
        \right\rfloor \epsilon 
        + (i-1) \epsilon
        \quad \text{ for } i = 1, 2, \dots, m. 
\end{equation}
That is, $[0, b_1), \dots, [1-\alpha - \epsilon, 1-\alpha),
[1-\alpha, 1-\alpha + \epsilon),
\dots , [b_m, b_{m+1} = 1]$ are $m+1$ disjoint intervals of $[0, 1]$, and each interval has a size of at most $\epsilon$.  See Figure \ref{fig:partition} for an illustration (in this example, $S_{H,0}$ is empty because $\frac{\alpha}{\epsilon}$ is an integer).

The idea in our subsequent analysis (leading to Corollary \ref{cor:samplebound_independent_H} below) is to bound the number of arms in each partition, and then bound the arm means in each partition by their ``worst-case'' values (i.e., those giving the smallest gaps).  The partition width $\epsilon$ is chosen such that the resulting ``rounding error'' coincides with the desired guarantee \eqref{eq:best_group_relax_eps_Delta} with parameter $\epsilon$.

\begin{figure}
    \begin{centering}
        \includegraphics[width=0.5\columnwidth]{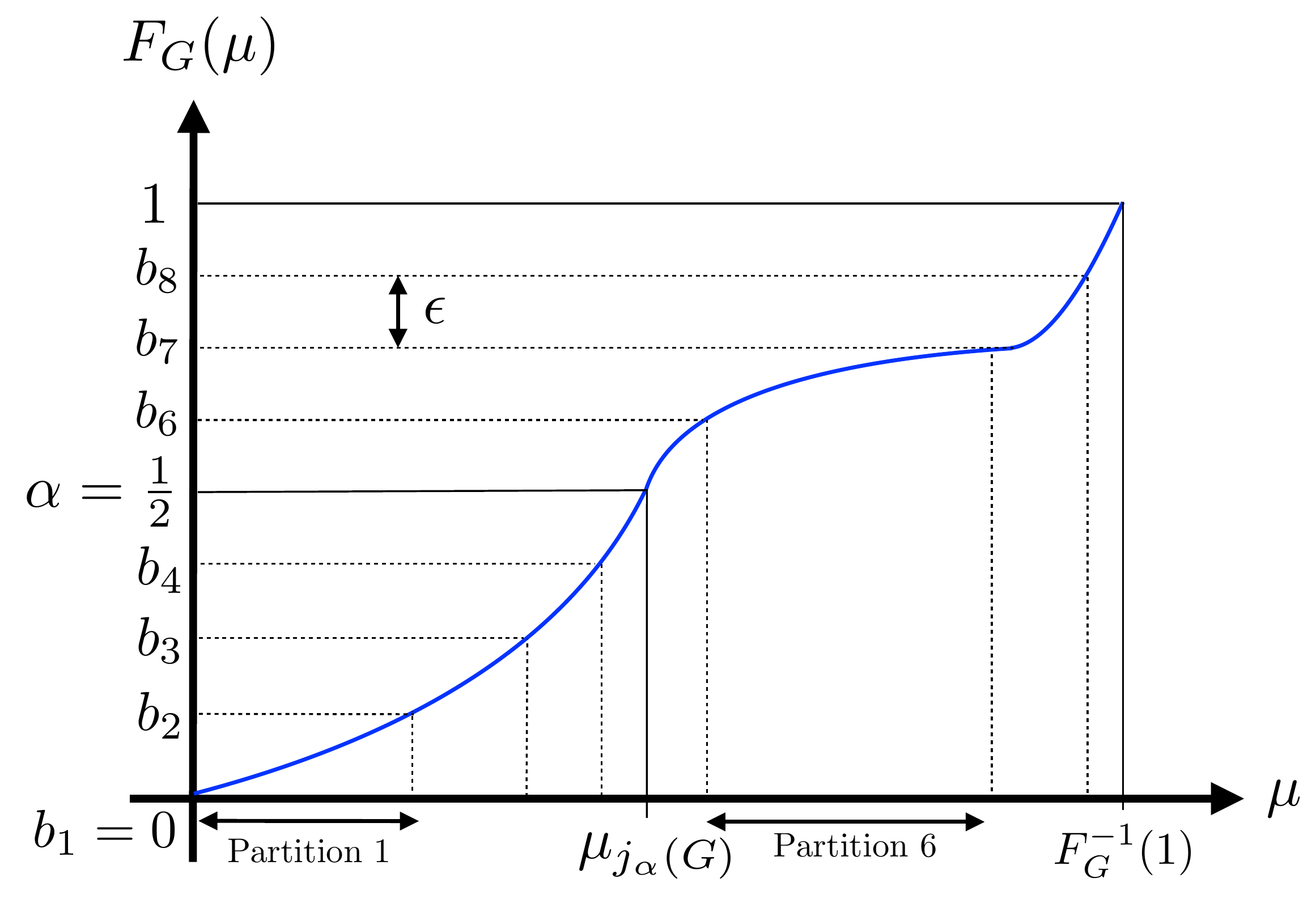}
        \par 
    \end{centering}
    
    \caption{Example of partitioning of arms, with $\alpha = \frac{1}{2}$ and $\epsilon = \frac{1}{8}$. \label{fig:partition}}
\end{figure}

Specifically, in Appendix \ref{app:DeltaH0}, we use the above partitioning to show that $\Delta'_{H, j} \ge \widetilde{\Delta}'_{G,i}$, where
\begin{equation}
    \widetilde{\Delta}'_{G,i}
    \coloneqq  
    \begin{cases}
        \ExpRewardInvCDF_{G}(1 - \alpha -\epsilon) - 
        \ExpRewardInvCDF_{G}(b_{i+1})
        & \text{ if } \, 0 \le i < 
        \left\lfloor \frac{1-\alpha}{\epsilon}
        \right\rfloor - 1\\
        \ExpRewardInvCDF_{G}(b_{i}) - 
        \ExpRewardInvCDF_{G}(1 - \alpha +\epsilon)
        & \text{ if } \,\left\lfloor \frac{1-\alpha}{\epsilon}
        \right\rfloor + 1 < i \le m \\
        0 & \text{otherwise}. 
    \end{cases} \label{eq:Delta_tilde}
\end{equation}
Intuitively, the third case corresponds to arms so close to the quantile that no positive $\Delta'_{H,j}$ gap can be guaranteed (though other gaps such as $\Delta_H$ may still be positive).

Combining the lower bounds on $\Delta_H$, $\Delta_0$, and $\Delta'_{H, j}$, we have
   $ \Delta_{H, j} 
    \ge
    \max\big\{
    \Delta,
    \widetilde{\Delta}_G,
    \widetilde{\Delta}_0,
    \widetilde{\Delta}'_{G,i}
    \big\}$,
where $i$ is the index of the partition of $H$ that $j$ belongs to, i.e., $j \in S_{H, i}$.
To simplify notation, let
\begin{equation}
\label{eq:tilde_Delta_G_i}
    \widetilde{\Delta}_{G, i} \coloneqq
    \widetilde{\Delta}_{G(H), i} =
    \max\left\{
    \Delta,
    \widetilde{\Delta}_G,
    \widetilde{\Delta}_0,
    \widetilde{\Delta}'_{G,i}
    \right\} \le \Delta_{H, j}.
\end{equation}
To completely remove the dependence on $\gH$, we also need to bound the number of arms in each partition $S_{H,i}$.  In Lemma \ref{lem:partition_arms} in Appendix \ref{app:intermediate}, we use a concentration argument to establish that $|S_{H,i}| \le 3\epsilon N$ for all $i$ with probability at least $1-\delta$, leading to our final guarantee stated as follows.
\begin{corollary}
    \label{cor:samplebound_independent_H}
    Let $(\gG, \alpha, \epsilon, \Delta, \delta)$ be a valid input of Algorithm~\ref{alg:main}.
    Under Assumption~\ref{as:noise}, with probability at least $1- 3\delta$, 
    Algorithm~\ref{alg:main} identifies a group $G \in \gG$
    satisfying \eqref{eq:best_group_relax_eps_Delta},
    and uses a total number of arm pulls satisfying
        \begin{align}
        \label{eq:armpullbound_independent_H}
            T(\epsilon, \delta, \Delta) 
            &\leq 
            \sum\limits_{G \in \gG}
            \sum_{i=1}^m
                \frac{c}{
            \widetilde{\Delta}_{G, i}^2}
           \log 
            \left(
                \frac{|\gG| N }{\delta}
                \log \frac{1}{\widetilde{\Delta}_{G, i}^2}
            \right) 
        \cdot
        3 \epsilon N   \\
        &\le  d \left(  
     \frac{|\gG|}{\epsilon^2 \Delta^2}  \right)
     \left(
     \log^2 \frac{|\gG|}{\delta}
     + 
     \left( \log \frac{|\gG|}{\delta} \right)
     \left( \log \log \frac{1}{\Delta} \right)
     \right) \label{eq:weakened_final}
        \end{align}
    where $c$ and $d$ are universal constants, $m$ is given in \eqref{eq:number_of_buckets},
    $\widetilde{\Delta}_{G, i}$ is as defined in \eqref{eq:tilde_Delta_G_i}, 
    and $N = \big \lceil\frac{1}{2\epsilon^2} \log \frac{2|\gG|}{\delta}\big\rceil$
    as stated in Line~\ref{line:N}
    of Algorithm~\ref{alg:main}.
\end{corollary}

The proof is deferred to Appendix~\ref{app:intermediate}.
Here \eqref{eq:armpullbound_independent_H} is our main instance-dependent bound, whereas \eqref{eq:weakened_final} is a weaker result obtained by lower bounding all gaps by $\Delta$, as was done in \eqref{eq:weakened} for the finite-arm setting (note that in both cases, the asymptotic simplifications rely on the assumption $\delta < \epsilon$ specified in Algorithm \ref{alg:main}).
In Section \ref{sec:lower}, we provide a near-matching algorithm-independent lower bound to \eqref{eq:weakened_final} in the worst case, and discuss the difficulties in obtaining instance-dependent lower bounds that nearly match \eqref{eq:armpullbound_independent_H}.  Moreover, in Appendix \ref{sec:comparisons}, we discuss some key differences compared to the problem of finding an arm in the top $(1-\alpha)$-quantile of a single group \cite{aziz2018pure}.


\vspace*{-2ex}
\subsection{Further Improvement via a Multi-Step Algorithm} \label{sec:improve}
\vspace*{-1ex}

A potential weakness of the upper bound in \eqref{eq:best_group_relax_eps_Delta} is that it is based on immediately choosing $N$ large enough such that each group's quantile is $\epsilon$-close to the true quantile in the sense of \eqref{eq:sampled_quantile_sandwiched} (i.e., Event~$A$).  However, if a group is highly suboptimal, it could potentially be identified as such even if $\epsilon$ is replaced by a much higher value in the choice of $N$ in Algorithm \ref{alg:main}.  Thus, if $\epsilon$ is very small, then immediately using it to choose $N$ (which has strong $\frac{1}{\epsilon^2}$ dependence) may be highly wasteful.

To alleviate this problem, one may use a \emph{sequence} of decreasing $\epsilon$ and $\Delta$ values.  We denote these by $\boldsymbol{\epsilon} = (\epsilon_1,\dotsc,\epsilon_K)$ and $\boldsymbol{\Delta} = (\Delta_1,\dotsc,\Delta_K)$ for some $K > 0$, and we assume that $(\epsilon_K,\Delta_K) = (\epsilon,\Delta)$.  We proceed in $K$ epochs, indexing them by $k$ (initially $k=1$) and proceeding as follows:
\begin{itemize}[itemsep=0pt,topsep=0pt,parsep=0pt]
    \item[(i)] Set $N_k = \big \lceil\frac{1}{2\epsilon_k^2} \log \frac{2|\gG|}{\delta}\big\rceil$, and request $N_k$ arms from each group.
    \item[(ii)] Run successive elimination (\texttt{FiniteArmBQID}) with the finite sets of arms from Step (i) and parameter $\Delta_k$.
    \item[(iii)] If multiple groups remain and $k < K$, then increment $k$ and return to Step (i) while permanently removing all groups that were eliminated in Step (ii) (as well as discarding all previously-requested arms).
    \item[(iv)] Otherwise, return $\widehat{G}$ corresponding to the $\widehat{H}$ returned in the last invocation of \texttt{FiniteArmBQID}.
\end{itemize}
In the following, we think of $K$ as having mild dependence with respect to the other parameters, e.g., if halving is used for the smaller of $\Delta$ and $\epsilon$, then $K = O\big( \log\frac{1}{\min\{\epsilon,\Delta\}} \big)$.

We proceed by outlining how the analysis and results are affected.  Since we are essentially running Algorithm \ref{alg:main} $K$ times, the error probability increases from $O(\delta)$ to $O(\delta K)$.  Since the dependence on $\delta$ is logarithmic, this effect of scaling $\delta$ is minimal.  More importantly, we need to identify when suboptimal groups are eliminated.  We let $\widetilde{\Delta}^{(\epsilon)}_{G}$ be defined as in \eqref{eq:Delta_H_lower_bound} with an explicit dependence on $\epsilon$, and introduce the following critical values:
\begin{itemize}[itemsep=0pt,topsep=0pt,parsep=0pt]
    \item If there exists $k$ such that $\widetilde{\Delta}^{(\epsilon_k)}_{G, i} > \Delta_k$, we let $k_{\max}(G)$ be the smallest such $k$ value.
    \item If no such $k$ exists, we let $k_{\max}(G) = K$.
\end{itemize}
Our analysis of Algorithm \ref{alg:main} reveals that a suboptimal group is eliminated when $\widetilde{\Delta}_{G} > \Delta$, and hence, the above definitions ensure that group $G$ will be eliminated (with high probability) after the end of epoch $k_{\max}(G)$.  As a result, we have the following generalization of Corollary \ref{cor:samplebound_simplified}.

\begin{corollary}
    \label{cor:samplebound_improved}
    Consider running the preceding multi-step algorithm with sequences $\boldsymbol{\epsilon} = (\epsilon_1,\dotsc,\epsilon_K)$ and $\boldsymbol{\Delta} = (\Delta_1,\dotsc,\Delta_K)$ for some $K > 0$, with $\delta_k < \epsilon_k < \min(\alpha,1-\alpha)$ for all $k$.  Under Assumption~\ref{as:noise}, with probability at least $1-3K\delta$, the algorithm identifies a group $G \in \gG$ satisfying~\eqref{eq:best_group_relax_eps_Delta},
    and uses a total number of arm pulls satisfying
        \begin{align}
            T(\epsilon, \delta, \Delta) 
            &\leq 
            \sum\limits_{G \in \gG} \sum_{k=1}^{k_{\max}(G)}
            \sum_{i=1}^{m_k}
                \frac{c}{
            (\widetilde{\Delta}^{(\epsilon_k)}_{G, i})^2}
           \log 
            \left(
                \frac{|\gG| N_k }{\delta}
                \log \frac{1}{(\widetilde{\Delta}^{(\epsilon_k)}_{G, i})^2}
            \right) 
        \cdot
        3 \epsilon_k N_k \label{eq:num_pulls_impr}
    \end{align}
    where $c$ is a universal constant, $m_k$ is as in \eqref{eq:number_of_buckets} with $\epsilon_k$ in place of $\epsilon$,
    $\widetilde{\Delta}^{(\epsilon)}_{G, i}$ is as in \eqref{eq:tilde_Delta_G_i} with $\Delta_G = \Delta_G^{(\epsilon)}$ explicitly depending on $\epsilon$, and $N_k = \big \lceil\frac{1}{2\epsilon_k^2} \log \frac{2|\gG|}{\delta}\big\rceil$.
\end{corollary}
This bound is more complicated than Corollary \ref{cor:samplebound_independent_H}, but can give significant improvements due to highly suboptimal groups being eliminated early and avoiding the smallest $\epsilon_k$ terms.  For example, in a scenario where all arms have the same gap term $\Delta$ but $\epsilon_{k_{\max}(G)} \gg \epsilon$ for all suboptimal $G$, the $\frac{1}{\epsilon^2}$ dependence in \eqref{eq:weakened_final} would be improved to $\frac{1}{(\min_{G \ne G^*} \epsilon_{k_{\max}(G)})^2}$.  See Appendix~\ref{sec:compare_aziz} for further discussion, and Figure \ref{fig:eps} (in Appendix \ref{sec:instance_lb}) for a visual illustration.

On the negative side, choosing the sequences $\boldsymbol{\epsilon}$ and $\boldsymbol{\Delta}$ may be tricky.  For example, if we choose $\Delta_k = \Delta$ for all $k$, then some of the larger $\epsilon_k$ values may lead to a group unnecessarily incurring $\frac{1}{\Delta^2}$ dependence (even if using the smallest $\epsilon$ would have led to a large gap -- again, see Figure \ref{fig:eps} in Appendix \ref{sec:instance_lb}).  One seemingly natural choice is to scale down each $\epsilon_k$ and $\Delta_k$ by a constant factor on each iteration, with the smaller of $\epsilon$ and $\Delta$ using a constant of $2$ (and the other constant being chosen to ensure the correct $K$-th value).  However, we leave it for future work as to whether this choice has near-optimality guarantees, or whether better universal choices exist.

%% file: lower_bounds.tex
\label{sec:lower}

Our final upper bounds on the number of arm pulls, given in Corollary \ref{cor:samplebound_independent_H} and Corollary \ref{cor:samplebound_improved}, have several dependencies of interest, including the number of partitions $m$, the gaps $\widetilde{\Delta}_{G, i}$, and the multiplicative term $3\epsilon N = O\big( \frac{1}{\epsilon} \log\frac{|\mathcal{G}|}{\delta} \big)$.  We discuss some of these dependencies in Section~\ref{sec:improve} and Appendix~\ref{sec:compare_aziz}, and we further discuss the difficulties of obtaining instance-dependent lower bounds in Appendix \ref{sec:instance_lb}.

In addition, we derived the weakened upper bound \eqref{eq:weakened_final}, which can also be viewed as corresponding to \eqref{eq:armpullbound_independent_H} in the case that all $\widetilde{\Delta}_{G, i}$ are on the same order as $\Delta$, i.e., a ``worst-case'' instance.  Observe that when $\delta = \frac{\epsilon}{2}$ (recall that we constrain $\delta < \epsilon$), \eqref{eq:weakened_final} reduces to\footnote{The notation $\widetilde{O}(\cdot)$ hides a $\log^2\frac{|\mathcal{G}|}{\epsilon}  + \log\frac{|\mathcal{G}|}{\epsilon} \cdot \log\log\frac{1}{\Delta}$ factor.  If we are interested in constant error probability, we could slightly modify our analysis to use two values $\delta_1$ and $\delta_2$ to replace the two occurrences of $\delta$ in Lemmas \ref{lem:EventA} and \ref{lem:EventB}.  Then we would only require $\delta_2 < \epsilon$, and the hidden dependence would improve to $\log\frac{|\mathcal{G}|}{\epsilon} \cdot \log\log\frac{1}{\Delta}$.}
\begin{equation}
	T(\epsilon,\delta,\Delta) \le \widetilde{O}\bigg( \frac{ |\mathcal{G}| }{ \Delta^2 \epsilon^2 } \bigg). \label{eq:ub_worst}
\end{equation}

Our main result of this section is the following, which states a near-matching $\Omega\big( \frac{1}{ \Delta^2 \epsilon^2 } \big)$ lower bound for worst-case instances.


\begin{theorem} \label{thm:lb}
	{\em (Worst-Case Lower Bound)}
	For any $\epsilon,\Delta \in \big(0,\frac{1}{4}\big)$, any given number of groups $|\mathcal{G}| \ge 2$,\footnote{We thank an anonymous reviewer for suggesting that we generalize this result from $|\mathcal{G}| = 2$ to general $|\mathcal{G}| \ge 2$.} any constant value of $\delta \in \big(0,\frac{1}{2}\big)$, and any algorithm for our max-quantile grouped bandit problem with $\alpha = \frac{1}{2}$ guaranteeing \eqref{eq:best_group_relax_eps_Delta} with probability at least $1-\delta$, there exists an instance with Bernoulli rewards such that the time horizon $T$ satisfies
	\begin{equation}
		\EE[T] \ge \Omega\bigg( \frac{ |\mathcal{G}| }{ \Delta^2 \epsilon^2 } \bigg). 
		\label{eq:lb_worst}
	\end{equation}
	That is, the worst-case upper bound in \eqref{eq:ub_worst} is tight up to logarithmic factors.
\end{theorem}

The proof is given in Appendix \ref{app:pf_lower}, and is based on a direct analysis of pairwise log-likelihood ratios between suitably-chosen ``hard-to-distinguish'' instances (having differing optimal groups) with Bernoulli rewards.  With $\alpha = \frac{1}{2}$, these instances are roughly defined as follows:
\begin{itemize} \itemsep0ex
	\item In all groups except the first, there are only two types of arms, ``good'' and ``bad'', with means $\frac{1}{2}+O(\Delta)$ and $\frac{1}{2}-O(\Delta)$ respectively.
	\item We let group 1 consist entirely of arms with mean $\frac{1}{2}$.  The remaining groups are either ``good'' with a $\frac{1}{2}+O(\epsilon)$ fraction of good arms, or ``bad'' with a $\frac{1}{2}-O(\epsilon)$ fraction of good arms.  We consider one instance where groups $2,\dotsc,|\mathcal{G}|$ are all bad (and hence group 1 is optimal), as well as $|\mathcal{G}|-1$ additional instances where only a single group $j \in \{2,\dotsc,|\mathcal{G}|\}$ is good (and hence group $j$ is optimal).  
\end{itemize}
Then, the rough idea is that for each group indexed by $j \in \{2,\dotsc,|\mathcal{G}|\}$, we need to consider $\Omega\big(\frac{1}{\epsilon^2}\big)$ arms from the group to determine whether or not it is a good group, but we also need to pull those arms $\Omega\big( \frac{1}{\Delta^2} \big)$ times each to determine whether they are good or bad arms, leading to $\Omega\big( \frac{1}{\Delta^2 \epsilon^2} \big)$ arm pulls per group.  Our formal proof makes this intuition more precise.

Naturally, it would also be of significant interest to develop instance-dependent lower bounds.  In Appendix \ref{sec:instance_lb}, we discuss some difficulties in doing so, including how certain standard change-of-measure techniques (e.g., see \cite{kaufmann2016complexity}) appear to be insufficient.  Despite this, such techniques do suggest that certain gap quantities may be relevant, and in Appendix \ref{sec:instance_lb} we compare those with the analogous quantities from our upper bound.

%% file: appendix_ub.tex
\section{
Proof of Theorem~\ref{thm:correctness}
(Correctness)}


%
%
\label{sec-correctness}

In Section \ref{sec:crrect_main}, we introduced Events A and B.  We now show them to be sufficient for Algorithm~\ref{alg:main} to return a group $G$
satisfying \eqref{eq:best_group_relax_eps_Delta}, and then show that each of these events happen with probability at least $1-\delta$ (see Lemma~\ref{lem:EventA} and~\ref{lem:EventB}). This implies that Algorithm~\ref{alg:main} returns a group $G$
satisfying \eqref{eq:best_group_relax_eps_Delta} with probability at least $1-2 \delta$, as stated in Theorem~\ref{thm:correctness}.

\begin{lemma}
    \label{lem:EventA}
    Event A occurs with probability at least $1-\delta$.
\end{lemma}
\begin{proof}
    This follows from union bounds and Chernoff-Hoeffding bounds; see Appendix~\ref{sect:appendix_infinite} for the details.
\end{proof}

\begin{lemma}
\label{lem:EventB}
There exists a choice of \texttt{FiniteArmBQID} (see Algorithm \ref{alg:elimination} in Appendix~\ref{sec:subroutine_details}) such that under Assumption~\ref{as:noise}, Event B occurs with probability at least $1-\delta$.
\end{lemma}
\begin{proof}
This is a corollary of Theorem~\ref{thm:ub_se} in Appendix~\ref{sec:subroutine_details}, whose proof is given in Appendix~\ref{sec:appendix_finite_alg}.
\end{proof}

Observe that under the event $A \cap B$, the group 
$\widehat{G}$ returned in Line~\ref{line:output_mainalg} of Algorithm~\ref{alg:main} satisfies 
\eqref{eq:best_group_relax_eps_Delta}, since
\begin{equation}
\ExpRewardInvCDF_{\widehat{G}}(1 - \alpha + \epsilon)
\ge \ExpRewardInvCDF_{\widehat{H}}(1 - \alpha)
\ge
\max_{H \in \gH} \ExpRewardInvCDF_{H}(1 - \alpha) - \Delta
\ge
\max_{G \in {\cal G}} \ExpRewardInvCDF_{G}(1 - \alpha - \epsilon) - \Delta. \label{rem:A_B_implies_correct}
\end{equation}
Combining Lemma~\ref{lem:EventA}, Lemma~\ref{lem:EventB}, and
\eqref{rem:A_B_implies_correct}
yields the correctness of Algorithm~\ref{alg:main}, as stated in Theorem~\ref{thm:correctness}.

\input{appendix_infinite.tex}

%% file: appendix_infinite.tex
\subsection{Proof of Lemma~\ref{lem:EventA} (Bounding the Probability of Event A)}
\label{sect:appendix_infinite}


We show that for a fixed $G \in \gG$ and the corresponding finite group $H$, we have
\begin{equation}
\label{eq:quantile_sampled_eps}
     \ExpRewardInvCDF_{G}(1 - \alpha - \epsilon)
     \le \ExpRewardInvCDF_{H}(1 - \alpha)
     \le \ExpRewardInvCDF_{G}(1 - \alpha + \epsilon).
\end{equation}
with probability at least $1 - \delta/|\gG|$. Then a union bound over 
the groups $G \in \gG$ gives us the desired claim.

The complement of event \eqref{eq:quantile_sampled_eps} is equivalent to at least one of these two events happening:

(i) $\ExpRewardInvCDF_{H}(1 - \alpha) > \ExpRewardInvCDF_{G}(1 - \alpha + \epsilon)$, or

(ii)
$\ExpRewardInvCDF_{H}(1 - \alpha) < \ExpRewardInvCDF_{G}(1 - \alpha - \epsilon)$.

We will upper-bound the probability of each of Event (i) and (ii) by $\frac{\delta}{2|\gG|}$, and then apply a union bound.

Let the number of sampled arms be as in Line~\ref{line:N} of Algorithm~\ref{alg:main}, i.e.,    $N =
    \big\lceil
    \frac{1}{2 \epsilon^2} \log \frac{2|\gG|}{\delta}
    \big\rceil.
$
Note that Event (i) is equivalent to more than $N \alpha$ of the sampled
arms being from the top-$(1-\alpha+\epsilon)$ quantile.
For $i =1, \dots, N$, we let $X_i = 1$ if the $i$-th sampled arm is from the top-$(1-\alpha+\epsilon)$ quantile, and $X_i = 0$ otherwise.
Then $X_1, \dots, X_N$ are i.i.d.~Bernoulli random variables with success probability at most $ \alpha - \epsilon$, and $X = \sum_{i=1}^N X_i$ is the number of sampled arms in the 
top-$(1-\alpha+\epsilon)$ quantile, with $\E[X] \le N(\alpha - \epsilon)$.
By Hoeffding's inequality, we obtain
\[
    \mathbb{P}[X \ge N\alpha] =
    \mathbb{P}[X - (N(\alpha - \epsilon)) \ge N\epsilon] 
    \le
    \mathbb{P}[X - \E[X] \ge N\epsilon]
    \le 
    \mathrm{exp}\left( -2 N \epsilon^2 \right)
    \le \frac{\delta}{2 |\gG|}.
\]
Similar reasoning shows that the probability of 
 Event (ii) is also bounded above by $\frac{\delta}{2|\gG|}$, which completes 
 the proof.

%% file: appendix_bound_independent_H.tex
\section{Proof of Corollary \ref{cor:samplebound_independent_H} (Number of Arm Pulls Without $\gH$ Dependence)}

\subsection{Derivations of Weakened Gaps} \label{app:DeltaH0}

Here we provide the derivations of the gap lower bounds stated in \eqref{eq:Delta_H_lower_bound}, \eqref{eq:Delta_0_lower_bound} and \eqref{eq:Delta_tilde}.  
Under Event~A, we can lower bound  each $\Delta_H$ by observing that
\begin{align}
    \widetilde{\Delta}_G \coloneqq
    \max_{G' \in \gG}  
    \ExpRewardInvCDF_{G'}(1 - \alpha - \epsilon)
    -  \ExpRewardInvCDF_{G}(1 - \alpha + \epsilon)
    &\le
    \max_{H' \in \gH}  
    \ExpRewardInvCDF_{H'}(1 - \alpha)  
    -  \ExpRewardInvCDF_{H}(1 - \alpha) \nonumber \\
    &=
    \ExpRewardInvCDF_{H^*}(1 - \alpha) 
    -  \ExpRewardInvCDF_{H}(1 - \alpha) =
    \Delta_H.
\end{align}
Similarly, we define the following value $\widetilde{\Delta}_0$,
which lower bounds $\Delta_0$ under Event A:
\begin{align}
    \widetilde{\Delta}_0 
    &\coloneqq \max\limits_{G \in \gG} \ExpRewardInvCDF_{G}(1 - \alpha - \epsilon)
    - \max\limits_{G'\in \gG, G' \ne G_{\epsilon}^*} \ExpRewardInvCDF_{G'}(1 - \alpha + \epsilon),
\end{align}
where 
$G_{\epsilon}^* \in \argmax_{G\in \gG} \ExpRewardInvCDF_{G}(1 - \alpha + \epsilon)$.  Indeed, we have under Event A that
\begin{equation}
    \max\limits_{G \in \gG}\ExpRewardInvCDF_{G}(1 - \alpha - \epsilon)
    - \max\limits_{G' \in \gG, G' \ne G_{\epsilon}^*}   \ExpRewardInvCDF_{G'}(1 - \alpha + \epsilon)
    \le
    \ExpRewardInvCDF_{H^*}(1 - \alpha)
    -  \max_{H \in \gH, H \ne H^*} \ExpRewardInvCDF_{H}(1 - \alpha)
    = \Delta_0,
\end{equation}
where the inequality uses the definition of $H^*$ (first term) and the fact that we subtract the second-highest value of a smaller quantity (second term).  Note that $\widetilde{\Delta}_0$ can be negative (e.g., when $\Delta_0 = 0$), but this is not a problem since we will later take the maximum with other gaps (see \eqref{eq:tilde_Delta_G_i}).

It remains to derive \eqref{eq:Delta_tilde}.   For each $0 \le i < \lfloor (1-\alpha)/\epsilon \rfloor - 1$,
and for each $j \in S_{H, i}$, we have the following lower bound:
\begin{align}
    \Delta'_{H, j} &=
    \left| \mu_j - \ExpRewardInvCDF_{H}(1 - \alpha) \right|  \\
    &=
    \ExpRewardInvCDF_{H}(1 - \alpha) - \mu_j \\
    &\ge
    \label{eq:Delta_H_j_lower_bound_left}
    \ExpRewardInvCDF_{G}(1 - \alpha -\epsilon) - 
    \ExpRewardInvCDF_{G}(b_{i+1}),
\end{align}    
where \eqref{eq:Delta_H_j_lower_bound_left} follows from
\eqref{eq:sampled_quantile_sandwiched} and the definition of $S_{H,i}$.
Likewise, for each $\lfloor (1-\alpha)/\epsilon \rfloor + 1 < i \le m$,
and for each $j \in S_{H, i}$, we must have
\begin{align}
    \Delta'_{H, j} &=
    \left| \mu_j - \ExpRewardInvCDF_{H}(1 - \alpha) \right| \\ 
    &=
    \mu_j - \ExpRewardInvCDF_{H}(1 - \alpha) \\
    &\ge
    \label{eq:Delta_H_j_lower_bound_right}
    \ExpRewardInvCDF_{G}(b_{i}) -
    \ExpRewardInvCDF_{G}(1 - \alpha +\epsilon).
\end{align}  
Summarizing the above, we obtain the desired result that $\Delta'_{H, j} \ge \widetilde{\Delta}'_{G,i}$ with $\widetilde{\Delta}'_{G,i}$ defined in \eqref{eq:tilde_Delta_G_i} (the third case therein is trivial).

\subsection{Intermediate Results and Proof of Corollary \ref{cor:samplebound_independent_H} (Number of Arm Pulls)} \label{app:intermediate}

Recall the lower bound on the gaps introduced in \eqref{eq:tilde_Delta_G_i}, which ensures that every arm $j$ in the same subset $S_{H,i}$ shares a common lower bound $\Delta_{H,j} \ge \widetilde{\Delta}_{G, i}$. 
Summing over the groups $H \in \gH$ and the subsets
$S_{H, i}$ in \eqref{eq:armpullboundrandom}, we obtain an upper bound of the following form.
\begin{corollary}
    \label{cor:samplebound_simplified}
    Under Assumption~\ref{as:noise}, with probability at least $1- 2\delta$, 
    the total number of arm pull used by Algorithm~\ref{alg:main} satisfies
    \begin{equation}
        \label{eq:armpullbound_simplified}
        T(\epsilon, \delta, \Delta) 
        \leq 
        \sum\limits_{H \in \gH}
        \sum_{i=1}^m
        \frac{c}{
            \widetilde{\Delta}_{G(H), i}^2}
        \log 
        \left(
        \frac{|\gH| N }{\delta}
        \log \frac{1}{\widetilde{\Delta}_{G(H), i}^2}
        \right)    
        \cdot
        \left| S_{H,i} \right|,
    \end{equation}
    where $c$ is a universal constant, $m$ is as defined in \eqref{eq:number_of_buckets},
    $\widetilde{\Delta}_{G(H), i}$ is as defined in \eqref{eq:tilde_Delta_G_i}, 
    and $S_{H,i}$ is as defined in \eqref{def:S_H_i}.
\end{corollary}
While the bound in
\eqref{eq:armpullbound_simplified}
does not contain the terms $\Delta_{H,j}$,
it is still dependent on the specific realization of $H$
through the terms $|S_{H,i}|$.
To remove this remaining dependence on $\gH$, we will show that with high probability,
$|S_{H,i}| \le 3 \epsilon N$ for each $H$ and $i$.
In other words, when drawing arms from $G$ to form $H$, with high probability, at most $3 \epsilon N$ arms are from the interval $[b_i, b_{i+1})$.

\begin{lemma}
    \label{lem:partition_arms}
    Let $(\gG, \alpha, \epsilon, \Delta, \delta)$ be a valid input of Algorithm~\ref{alg:main}.
    Let $N$ be as in Line~\ref{line:N} of Algorithm~~\ref{alg:main}, and let $m$
    be as in \eqref{eq:number_of_buckets}.
    For each $G \in \gG$ and its random arm group $H = H(G)$ generated in Lines~\ref{line:forloop}--\ref{line:generate_H} of Algorithm~\ref{alg:main}, we partition $H$
    into $m+1$ disjoint multisets $\{  S_{H, i} \}_{i=0}^m$
    as defined in \eqref{def:S_H_i}.
    Then with probability at least $1-\delta$,
    we have $|S_{H, i}| \le 3 \epsilon N = O\big( \frac{1}{\epsilon} \log\frac{|\mathcal{G}|}{\delta} \big)$
    for each $H \in \gH$ and for each $i = 0, \dots, m$.
\end{lemma}
\begin{proof}
    This follows easily from Hoeffding's inequality and the union bound; see Appendix~\ref{sec:appendix_partition} for the details.
\end{proof}
Combining Theorem~\ref{thm:correctness}, Corollary~\ref{cor:samplebound_simplified}
and Lemma~\ref{lem:partition_arms}, we obtain Corollary \ref{cor:samplebound_independent_H} as desired.

\input{appendix_partitioning}

%% file: appendix_partitioning.tex
\subsection{Proof of Lemma~\ref{lem:partition_arms} (Partitioning of Arms)} 
\label{sec:appendix_partition}

For each $H \in \gH$, and $i = 0,1, \dots, m$, we let $E_{H,i}$ denote the event that $|S_{H, i}| \le 3 \epsilon N$, where we recall the choice
\begin{equation}
    N =
    \left\lceil
    \frac{1}{2 \epsilon^2} \log \frac{2|\gG|}{\delta}
    \right \rceil
    =
    \left\lceil
     \frac{1}{2 \epsilon^2} \log \frac{2|\gH|}{\delta}
     \right\rceil. \label{eq:choiceN}
\end{equation}
We will show that the complement of each event $E_{H,i}$ occurs with probability at most $\frac{\delta}{(m+1)|\gH|}$. Then, a union bound over all $H \in \gH$  and $i = 0, \dots, m$ gives us the desired claim.

Fix an arbitrary $H \in \gH$ and an arbitrary $i \in \{0, \dots, m\}$.
By the definition in   \eqref{def:S_H_i}, the complement of event $E_{H,i}$ is equivalent to more than $3 \epsilon N$ of the $N$ arms independently sampled uniformly from $[0,1]$ 
satisfying $j \in [b_{i},  b_{i+1})$, where $b_0 \coloneqq 0, b_{m+1} \coloneqq 1,$ and $b_1, \dots b_m$ are as defined in \eqref{eq: b_i}.
Each such arm is in $[b_{i},  b_{i+1})$ with probability at most $b_{i+1} - b_i \le \epsilon$.

For $k = 1, \dots, N$, we let $X_k = 1$ if the $k$-th arm sampled from $G$ is from the subset $[b_{i},  b_{i+1})$, and $X_k = 0$ otherwise.
Then $X_1, \dots, X_N$ are i.i.d. Bernoulli Random Variables with success probability 
at most $\epsilon$, and
$X = \sum_{k=1}^N X_k = |S_{H,i}|$, with $\E[X] \le N \epsilon$.
By Hoeffding's inequality, we have
\[
    \mathbb{P}[X > 3N\epsilon] 
    \le
    \mathbb{P}[X - \E[X] \ge 2 N \epsilon]
    \le 
    \mathrm{exp}\left( -8 N \epsilon^2 \right)
    \le \frac{\delta^4}{16 |\gH|^4}
    < \frac{\delta}{(m+1) |\gH|},
\]
where the last two steps use the choice of $N$ in \eqref{eq:choiceN}, along with
$\delta^3 < \delta < \epsilon \le \frac{1}{m-1} < \frac{16 |\gH|^3}{m+1}$.



%% file: appendix_subroutine_finite.tex
\section{Details of Subroutine for the Finite-Arm Setting}
\label{sec:subroutine_details}

In this section, we formally define our choice of \texttt{FiniteArmBQID} used 
in Line~\ref{line:finite_alg} of Algorithm~\ref{alg:main}.  A complete description of our finite-arm subroutine will be given in Algorithm~\ref{alg:elimination} below.  
The algorithm maintains a set of active arms and groups 
(see \eqref{def:C_t}, \eqref{def:m_t}, \eqref{def:B_t}, and Lines~8-10 of Algorithm~\ref{alg:elimination}).
At each round, the algorithm pulls all arms in the set of active arms, updates their confidence bounds, and eliminates groups that are suboptimal and arms that are ``no longer of interest'' based on the confidence bounds.
When the algorithm identifies that
some group satisfies 
\eqref{eq:optimal H} based on the confidence bounds, it terminates and returns that group.

While the algorithm and results in this appendix are used as a stepping stone to the overall guarantees of Algorithm \ref{alg:main}, we believe that they are also of interest in their own right.

\subsection{Integer-Valued Indexing} \label{sec:indexing}

Since the number of arms in $\gH$ is finite, we \textit{re-index} the set of all arms in $\gH$
by $\{1, \dots, n\}$, where  
\begin{equation}
    \label{eq:number_of_finite_arms}
        n \coloneqq  
        \left| \bigcup\limits_{H \in \gH} H \right|
        = N|\gH| = \left \lceil
        \frac{1}{2\epsilon^2} \log \frac{2|\gG|}{\delta} 
        \right\rceil \cdot |\gG|
\end{equation}
recalling the choice of $N$ in Algorithm \ref{alg:main}.  
We will mostly use this integer-valued indexing $j \in \{1,\dotsc,n\}$ for arms in this section and its associated appendices.  If $j \in \{1,\dotsc,n\}$ is the arm's integer-valued index, $j' \in (0,1)$ is the arm's index in $[0,1]$, and $H$ is the arm's group, then we adopt the shorthand notation
\begin{equation}
    \Delta_j = \Delta_{H,j'}. \label{eq:shorthand}
\end{equation}
We will also slightly abuse notation and write $j \in H$ and $j' \in H$ interchangeably under the two forms of indexing; it will be clear from the context when $j$ is an integer index vs.~a continuous index in $[0,1]$.

\subsection{Law of the Iterated Logarithm and Confidence Bounds} \label{sec:aux}

As is ubiquitous in MAB problems, our analysis relies on confidence bounds.  Despite our distinct objective, our setup still consists of regular arm pulls, and accordingly, we can utilize well-established confidence bounds for stochastic bandits.  Many such bounds exist with varying degrees of simplicity vs.~tightness, and for concreteness, we focus on the law of the iterated logarithm (as refined by \cite[Theorem 8]{kaufmann2016complexity}; see also \cite[Lemma 3]{jamieson2014lil}).
This result shows that
with high probability,
for each arm $j$ and each round index $t$, the mean reward~$\mu_j$ 
is within some confidence interval
whose width decreases as $T_j(t)$ increases.


To formally present the result, we introduce and recall some notations and definitions.  
For each $j \in \{1, \dots, n\}$ and each round index $t \ge 1$, let 
$\widehat{\mu}_{j, T_j(t)}$ denotes the empirical mean of observed rewards of arm $j$ up to round $t$. 
We will use the following function  $U \colon \mathbb{Z}^+ \times (0, 1)
\to \R^+$:
\begin{equation}
\label{eq:confidence_width}
    U(T, \delta)
    \coloneqq
    \sqrt{
    \frac{2 \log(1/\delta) 
    + 6 \log \log (1/\delta) 
    + 3 \log \log(eT)}
    {T}
    }
    = \Theta\left( 
    \sqrt{
    \frac{1}{T}
    \log \left(\frac{\log T}{\delta}\right)
    }
    \right),
\end{equation}
which describes the width of the confidence interval of a mean reward after $T$ pulls.

\begin{lemma}[Anytime confidence bounds {\cite[Theorem 8]{kaufmann2016complexity}}]
\label{lem:cbs}
    Given $\delta \in (0,1)$, under Assumption~\ref{as:noise}, we have
    with probability at least $1-\delta$ that
    \begin{equation}
    \label{eq:cbs}
    \mathrm{LCB}_t(j) 
    \leq 
    \mu_j \leq \mathrm{UCB}_t(j)
    \quad 
    \text{for all arms } j \in \{
    1, \dotsc, n\}
    \text{ and }
    \text{for all rounds } t \ge 1,
    \end{equation}
    where
    \begin{equation}
    \label{eq:lcb_ucb} 
    \mathrm{LCB}_t(j)
    \coloneqq
    \widehat{\mu}_{j, T_j(t)} - 
    U\Bigg(T_j(t), \frac{\delta}{n}\Bigg)
    \quad
    \text{and}
    \quad
    \mathrm{UCB}_t(j)
    \coloneqq \widehat{\mu}_{j, T_j(t)} + 
    U\Bigg(T_j(t), \frac{\delta}{n}\Bigg)
    \end{equation}
    are the lower and upper confidence bounds for 
    arm $j$ at round $t$.
\end{lemma}

\subsection{Successive Elimination} 
\label{sec:se}

With the confidence bounds \eqref{eq:cbs}
formally defined, we  now formally describe how Algorithm~\ref{alg:elimination} maintains the active arms and groups at round $t$ based on confidence bounds.
In the following, we write 
$\mathrm{Q}_{1-\alpha}(X)$
as the $(1-\alpha)$-quantile of a finite multiset $X$, defined in the same way as \eqref{eq:forwardH}.

We will work in rounds indexed by $t \ge 1$. 
We first define the set $\gC_t$ of candidate {\em potentially optimal groups} at the beginning of round $t$, initialized as ${{\gC}_1} \coloneqq \gH$ and subsequently updated based on confidence bounds \eqref{eq:cbs} as follows:
\begin{equation} \label{def:C_t}
    {\gC}_{t+1} \coloneqq 
    \Bigg\{H \in {\gC}_t \Bigm| \       
    \mathrm{Q}_{1-\alpha}
    \left(
    \left\{
    \mathrm{UCB}_{t}(j) : 
    j \in H
    \right\}
    \right)  
    \geq
    \mathrm{Q}_{1-\alpha}
    \left(
    \left\{
    \mathrm{LCB}_{t}(j') : 
    j' \in H'
    \right\}
    \right)  
    \ \forall H' \in {\gC}_t\Bigg\}.
\end{equation}
This definition allows us to eliminate groups that are suboptimal according to the confidence bounds.
For each group $H \in \gC_t$, we define the set $m_t^{(H)}$ of \textit{potential $(1-\alpha)$-quantile arms} of $H$
at the beginning of round $t$, initialized as $m_1^{(H)} \coloneqq H$, and subsequently updated based on the confidence bounds as follows:
\begin{align} \label{def:m_t}
    m_{t+1}^{(H)}\coloneqq 
    \Big\{j \in H 
    \bigm| \
    & \mathrm{LCB}_{t}(j) \leq  
    \mathrm{Q}_{1-\alpha}
    \left(
    \left\{
    \mathrm{UCB}_{t}(j') : 
    j' \in H
    \right\}
    \right)  \ 
    ~~\text{and} 
    \\ \notag
    & \mathrm{UCB}_{t}(j) \geq 
    \mathrm{Q}_{1-\alpha}
    \left(
    \left\{
    \mathrm{LCB}_{t}(j') : 
    j' \in H
    \right\}
    \right)  
    \Big\}.
\end{align}
This definition allows us to eliminate arms that are no longer potentially $(1-\alpha)$-quantile arms in $H$ according to confidence bounds \eqref{eq:cbs}.
Based on \eqref{def:C_t} and \eqref{def:m_t}, we define the set $\gB_t$ of active arms, i.e., the arms that will be pulled in round $t$, initialized as $\gB_1 \coloneqq \bigcup\limits_{H \in \gH} H =\{1, \dots, n\}$, and subsequently updated as follows:
\begin{align} 
\label{def:B_t}
    {\gB}_{t+1} \coloneqq \; \Big\{j \bigm| j \in m_{t+1}^{(H)} \;  \text{for some } H \in {\gC}_{t+1}\Big\}.
\end{align}

Recall that we only need to find (with high probability) a group $H$ whose quantile is within $\Delta$ of the highest, i.e., satisfies \eqref{eq:optimal H}. To take advantage of this relaxation, we define the quantity
$\Delta^{(t)}$ that tracks the
difference between the most optimistic  and pessimistic possibilities of
$\max\limits_{H \in \gC_{t}} \ExpRewardInvCDF_{H}(1 - \alpha) $ according to confidence bounds
at the beginning of round $t$, initialized as $\Delta^{(1)} \coloneqq \infty$, and subsequently updated as follows:
\begin{equation}
\label{eq:max_dist}
    \Delta^{(t+1)} \coloneqq
    \max_{H \in \gC_{t+1}} 
    \left\{
    \mathrm{Q}_{1-\alpha}
    \left(
    \left\{
    \mathrm{UCB}_{t}(j) : 
    j \in H
    \right\}
    \right)
    \right\}
    -
    \max_{H' \in \gC_{t+1}}
    \left\{
    \mathrm{Q}_{1-\alpha}
    \left(
    \left\{
    \mathrm{LCB}_{t}(j) : 
    j \in H'
    \right\}
    \right)
    \right\}.
\end{equation}
Once $\Delta^{(t)} \le \Delta$, each group 
$H \in \argmax\limits_{H \in \gC_{t}}
    \left\{
    \mathrm{Q}_{1-\alpha}
    \left(
    \left\{
    \mathrm{LCB}_{t-1}(j) : 
    j \in H
    \right\}
    \right)
    \right\}$ satisfies \eqref{eq:optimal H} provided that the confidence bounds are valid, and so we can simply return any one of them.

With these definitions in place, we now present the pseudo-code for Algorithm~\ref{alg:elimination}.

\begin{algorithm}
    \caption{Max-Quantile Grouped Finite-Arm Bandit Algorithm (Used as \texttt{FiniteArmBQID} in Algorithm \ref{alg:main})}\label{alg:elimination}
    \begin{algorithmic}[1]
        \Require~~ 
        A finite set of finite arm groups $\gH$, parameters 
        $\alpha, \Delta, \delta \in (0,1)$

        \State Initialize $t = 1$ and $T_j(0) = 0$ for all 
        $j \in \bigcup\limits_{H \in \gH} H = \{1, \dots, n\}$
        
        \State Set $m_1^{(H)} = H \;$ for all $H \in {\gH}$;
        set ${{\gC}_1} = \gH$, $\Delta^{(1)} = \infty$, and $\gB_1 = \bigcup\limits_{H \in \gH} H$
        
        \While {$|{\gC}_t| > 1$
         and $\Delta^{(t)} > \Delta$
        }
        
        \State Pull every arm $j \in \gB_t$ once 
        
        \State Update $T_j(t) = T_j(t-1) + 1$ 
        if $j \in \gB_t$ and $T_j(t) = T_j(t-1)$ otherwise
        \label{line:update_T}
        
        \State Compute $U\big(T_j(t),  \delta_{j, T_j(t)} \big)$ for every arm $j =\{1, \dots, n\}$ according to \eqref{eq:confidence_width}
        
        \State Compute $\mathrm{LCB}_t(j)$ and
        $\mathrm{UCB}_t(j)$ for every arm $j =\{1, \dots, n\}$  according to \eqref{eq:lcb_ucb}
        
        \State Compute ${\gC}_{t+1}$ according to  {\eqref{def:C_t}} \label{line:update_C_t}

     \State Compute $m_{t+1}^{(H)}$
        for each $H \in \gC_{t+1}$
     according to {\eqref{def:m_t}}

     \State Compute $\gB_{t+1}$ and $\Delta^{(t+1)}$  according to
        {\eqref{def:B_t}} and
        \eqref{eq:max_dist} 
        \label{line:update_B_t}
        
        \State Increment the round index $t$ by 1
        \EndWhile 
        
        \State Set $\widehat{H}  \in \argmax\limits_{H \in \gC_{t}}
    \left\{
    \mathrm{Q}_{1-\alpha}
    \left(
    \left\{
    \mathrm{LCB}_{t-1}(j) : 
    j \in H
    \right\}
    \right)
    \right\}$ with uniformly random tie-breaking
        
        \State \textbf{return} $\widehat{H}$
    \end{algorithmic}
\end{algorithm}

\begin{remark}
\label{rem: efficient implementation}
    While computational complexity is not our focus, we note that Lines \ref{line:update_T}--\ref{line:update_B_t} can be computed more
    efficiently than a direct implementation by observing the following:
    \begin{itemize}[topsep=0pt,itemsep=0pt]
        \item Every non-eliminated arm has been pulled the same number of times (i.e., $T_j(t) = t $ if $j \in \gB_t$), and hence,
        \begin{equation}
            U\bigg(T_j(t),  \frac{\delta}{n} \bigg) = U\bigg(t,  \frac{\delta}{n} \bigg) 
            \quad \forall j \in \gB_t,
        \end{equation}
        from which $\mathrm{LCB}_t(j)$ and $\mathrm{UCB}_t(j)$ can readily be computed.  Moreover, the fact that every non-eliminated arm has the same confidence width implies that
        \begin{equation}
            \Delta^{(t+1)} = 2 U\bigg(t,  \frac{\delta}{n} \bigg)
            \quad \forall t \ge 1,
        \end{equation}
        with both maxima in \eqref{eq:max_dist} being attained by the same group.
        \item After eliminating any arm, its UCB and LCB values no longer need to be computed.
    \end{itemize}
\end{remark}

We now state our main result regarding Algorithm~\ref{alg:elimination}.
\begin{theorem} \label{thm:ub_se}
    {\em (Performance Guarantee for Algorithm~\ref{alg:elimination})}
    Let $(\gH, \alpha, \Delta, \delta)$ be a valid input of Algorithm~\ref{alg:elimination}.
    Under Assumption~\ref{as:noise},
    with probability at least $1-\delta$,  Algorithm~\ref{alg:elimination} identifies a group $\widehat{H}$ satisfying 
    \eqref{eq:optimal H} and uses a number of arm pulls satisfying
    \begin{align}
        T(\delta) 
        \le
        \sum\limits_{j = 1}^{n} 
                 \frac{c}{\Delta_j^2}
       \log 
        \left(
            \frac{n}{\delta}
            \log \frac{1}{\Delta_j^2}
        \right)
        \label{eq:T_total}
    \end{align}
    where $c$ is a universal constant, $n$ is as defined in
    \eqref{eq:number_of_finite_arms}, and $\Delta_j$ is as defined in \eqref{def:gap}.
\end{theorem}
\begin{proof}
    The proof follows typical steps used in elimination-based algorithms for finite-arm settings (with \cite{wang2021max} being perhaps the most similar), but requires care in handling the variety of cases that may occur.  See Appendix~\ref{sec:appendix_finite_alg} for the details.
\end{proof}

Observe that Theorem \ref{thm:ub_se} is equivalent to the first statement in Theorem \ref{thm:samplebound}, with the latter using continuous indexing instead of integer indexing.



%% file: appendix_finite.tex
\section{Proof of Theorem~\ref{thm:ub_se} (Performance Guarantee for Algorithm~\ref{alg:elimination})} 
\label{sec:appendix_finite_alg}

Throughout this appendix, we use the integer-valued indexing convention introduced in Appendix~\ref{sec:indexing}, in particular using $\Delta_j$ as per \eqref{eq:shorthand}.

Recalling Lemma~\ref{lem:cbs}, it is sufficient to show that when the confidence bounds \eqref{eq:lcb_ucb} are valid (with parameter $\delta$),  
Algorithm~\ref{alg:elimination} identifies a group $\widehat{H}$ satisfying \eqref{eq:optimal H} and uses a number of arm pulls satisfying \eqref{eq:T_total}.
We separate this conditional performance guarantee into two parts: correctness in Lemma~\ref{lem:conv_SE}, and
bounding the number of arm pulls in 
Corollary~\ref{cor:cond_armpull}.

\begin{lemma}[Conditional Correctness of Algorithm~\ref{alg:elimination}] \label{lem:conv_SE}
    If the confidence bounds \eqref{eq:lcb_ucb}
    are valid, then Algorithm~\ref{alg:elimination} returns a group $\widehat{H}$ satisfying
    \eqref{eq:optimal H}.
\end{lemma}
\begin{proof}
    See Appendix~\ref{appendix:proof_cond_correctness}.
\end{proof}
For the conditional bound on arm pulls,
we show that if the confidence bounds \eqref{eq:lcb_ucb} are valid, then 
the condition $U\big(T_j(t), \frac{\delta}{n}  \big) < \frac{\Delta_j}{4}$ is sufficient to conclude that arm $j$ will not be pulled at rounds $\tau \ge t+1$.
This condition allows us to find an upper bound on the number of pulls of arm $j$.
Summing over all arms $j \in \{1, \dots, n\}$ yields the bound.

\begin{lemma}
\label{lem:stop_pull_condition}
     If the confidence bounds \eqref{eq:lcb_ucb} are valid, then for each arm $j \in \{1, \dots, n\}$ and any time~$t$, we have
  \[
    U\bigg(T_j(t), \frac{\delta}{n} \bigg) < \frac{\Delta_j}{4}
    \implies
    \left(
    j \not\in \gB_{t+1} \ \text{ or } \
    |\gC_{t+1}| = 1 \ \text{ or } \
    \Delta^{(t+1)} \le \Delta
    \right)
  \]
   That is, after the round index~$t$ satisfies 
   $U\big(T_j(t), \frac{\delta}{n} \big) < \frac{\Delta_j}{4}$,
   arm $j$ will no longer be pulled in Algorithm~\ref{alg:elimination}.
\end{lemma}
\begin{proof}
    See  
    Appendix~\ref{appendix:proof_suff_condition_stop}.
\end{proof}

To make the condition on $t$ more explicit, we write
\begin{equation} \label{eq:interval}
        \min\left\{t: 
        U\bigg(T_j(t), \frac{\delta}{n} \bigg) < \frac{\Delta_j}{4}\right\} 
        =
        \min\left\{t: 
        U\bigg(t, \frac{\delta}{n} \bigg) < \frac{\Delta_j}{4}\right\} 
        \le
         \frac{c}{(\Delta_j)^2}
       \log 
        \left(
            \frac{n}{\delta}
            \log \frac{1}{(\Delta_j)^2}
        \right),
    \end{equation}
where the last step holds for some universal $c > 0$ by a standard inversion, e.g., \cite[p.5]{JamiesonCSE599}. 
Hence, using Lemma \ref{lem:stop_pull_condition} and summing over the arms, we obtain the following.

\begin{corollary}[Conditional Bounds on Arm Pulls]
\label{cor:cond_armpull}
      If the confidence bounds \eqref{eq:lcb_ucb} are valid, then Algorithm~\ref{alg:elimination} uses a number of arm pulls satisfying
      \eqref{eq:T_total}.
\end{corollary}

\subsection{Proof of Lemma~\ref{lem:conv_SE} (Correctness of Algorithm~\ref{alg:elimination})}
\label{appendix:proof_cond_correctness}
For brevity, we say that an arm is a $(1-\alpha)$-quantile arm in group $H$ if it has a mean reward of $F^{-1}_H(1-\alpha)$, 
    and we denote an arbitrary such arm by $j_{\alpha}(H)$.\footnote{The use of infimum in \eqref{eq:forwardH} ensures that such an arm always exists.} We let
    $H^* \in \argmax\limits_{H \in \gH} \ExpRewardInvCDF_{H}(1 - \alpha) $ be a single optimal group, breaking ties arbitrarily in the case of non-uniqueness.

    We first show that $H^* \in \gC_t$ for each round $t \ge 1$, i.e., $H^*$ always remains a potentially optimal group.
    For each fixed $t \ge 1$, we let $\gE_t$ denote the event that $H^* \in \gC_t$.
    We show by induction that ${\gE}_t$ holds for all $t \ge 1$.
    For $t = 1$,
    we have $H^* \in \gH = \gC_1$. 
    We now show the inductive step: When ${\gE}_{t}$ holds, so does ${\gE}_{t+1}$.
    For all $H' \in {\gC}_{t}$, we have
    \begin{align}
       \mathrm{Q}_{1-\alpha}
        \left(
        \left\{
        \mathrm{UCB}_{t}(j) : 
        j \in H^*
        \right\}
        \right)  
        &\geq
        \mathrm{Q}_{1-\alpha}
        \left(
        \left\{
        \mu_{j} : 
        j \in H^*
        \right\} 
        \right) 
        \label{eq: G*_tn_C_1}        \\
        &= \mu_{j_{\alpha}(H^*)} \label{eq: G*_tn_C_2} \\
        &\geq \mu_{j_{\alpha}(H')} \label{eq: G*_tn_C_3} \\
        &= 
        \mathrm{Q}_{1-\alpha}
        \left(
        \left\{
        \mu_{j'} : 
        j' \in H'
        \right\}
        \right)  \\ 
        &\geq 
        \mathrm{Q}_{1-\alpha}
        \left(
        \left\{
        \mathrm{LCB}_{t}(j') : 
        j' \in H'
        \right\}
        \right)
        \label{eq: G*_tn_C_4} ,
    \end{align}
    where \eqref{eq: G*_tn_C_1} 
    and \eqref{eq: G*_tn_C_4}
    follow from the confidence bounds \eqref{eq:cbs}--\eqref{eq:lcb_ucb}, and
    \eqref{eq: G*_tn_C_3} uses the definition of $H^*$.
    By \eqref{eq: G*_tn_C_4} and the definition of ${\gC}_{t+1}$  (see \eqref{def:C_t}), we conclude that $H^* \in {\gC}_{t+1}$ as desired. 
    
    
    We now argue that the while-loop  of Algorithm~\ref{alg:elimination} will terminate, and
    the returned group $\widehat{H}$ satisfies~\eqref{eq:optimal H}.
    The halting criteria of while-loop will eventually be satisfied because the width of confidence intervals satisfies
    $U(T,\delta) \to 0$ as $T \to \infty$ for any $\delta > 0$ (see \eqref{eq:confidence_width}).
    If the while-loop of Algorithm~\ref{alg:elimination}
    terminates because $|{\gC}_t| = 1$,
    then $\gC_t = \{H^*\}$. It trivially follows that
    the returned group $\widehat{H} = H^*$ satisfies~\eqref{eq:optimal H}.
    On the other hand, if the while-loop  terminates because 
    $\Delta^{(t)} \le \Delta$ for some $t \ge 2$, then for an arbitrary 
    $\widehat{H} \in  \argmax\limits_{H \in \gC_{t}}
    \left\{
    \mathrm{Q}_{1-\alpha}
    \left(
    \left\{
    \mathrm{LCB}_{t-1}(j) : 
    j \in H
    \right\}
    \right)
    \right\}$, we have
    \begin{align}
    \ExpRewardInvCDF_{\widehat{H}}(1 - \alpha)
     &= \mathrm{Q}_{1-\alpha}
        \big(
        \big\{
        \mu_{j} : 
        j \in \widehat{H}
        \big\} 
        \big)\\  
    &\ge 
        \mathrm{Q}_{1-\alpha}
        \big(
        \big\{
        \mathrm{LCB}_{t-1}(j) : 
        j \in \widehat{H}
        \big\}
        \big) 
        \label{eq:terminate_line2} 
        \\ 
    & = \max_{H \in \gC_t} 
        \left\{
        \mathrm{Q}_{1-\alpha}
        \left(
        \left\{
        \mathrm{LCB}_{t-1}(j) : 
        j \in H
        \right\}
        \right)
        \right\} \\
    & = \max_{H \in \gC_t} 
        \left\{
        \mathrm{Q}_{1-\alpha}
        \left(
        \left\{
        \mathrm{UCB}_{t-1}(j) : 
        j \in H
        \right\}
        \right)
        \right\} - \Delta^{(t)} 
        \label{eq:max_dist_equivalent}\\
    & \ge \max_{H \in \gC_t} 
        \left\{
        \mathrm{Q}_{1-\alpha}
        \left(
        \left\{
        \mu_j : 
        j \in H
        \right\}
        \right)
        \right\} - \Delta
        \label{eq:while_loop_terminate_5}\\
    & =  \mathrm{Q}_{1-\alpha}
        \left( \left\{
        \mu_j : 
        j \in H^* \right\} \right)
        - \Delta \label{eq:while_loop_terminate_6} \\ 
    &=  \max_{H \in \gH} \ExpRewardInvCDF_{H}(1 - \alpha)   - \Delta,   
    \end{align}
    where 
    \eqref{eq:terminate_line2} and \eqref{eq:while_loop_terminate_5} follow from confidence bounds and $\Delta^{(t)} \le \Delta$,
    \eqref{eq:max_dist_equivalent} follows from 
    \eqref{eq:max_dist},
    and \eqref{eq:while_loop_terminate_6} follows
    from $H^* \in \gC_t$.
    Hence, $\widehat{H}$ satisfies \eqref{eq:optimal H} in both cases.



\subsection{Proof of Lemma~\ref{lem:stop_pull_condition} (Sufficient Conditions for No Longer Being Pulled)}
\label{appendix:proof_suff_condition_stop}

We first present a useful auxiliary lemma.  Observe that if $U\big(T_j(t), \frac{\delta}{n} \big) < \frac{\Delta_j}{4}$ and the confidence bounds are valid, we have
\begin{equation}
\label{eq:mu_j confidence bounds}    
    \max \Big\{ \mu_j - \mathrm{LCB}_{t}(j), \mathrm{UCB}_{t}(j) - \mu_j  \Big\}
    \le
    \mathrm{UCB}_{t}(j) - \mathrm{LCB}_{t}(j) = 
    2 \ U\bigg(T_j(t),  \frac{\delta}{n} \bigg) < \frac{\Delta_j}{2}. 
\end{equation}
We use this observation to prove the following.

\begin{lemma}
Let arm $j \in \{1, \cdots, n\}$ be arbitrary and let $H$ be
the group containing $j$.
Moreover, suppose that the confidence bounds are valid.
If the round index $t \ge 1 $ satisfies
$U\big(T_j(t),  \frac{\delta}{n} \big) < \frac{\Delta_j}{4}$, then we have
\begin{equation}
\label{eq:quantilebound}
     \ExpRewardInvCDF_{H}(1 - \alpha) + \frac{\Delta_j}{2} 
          >
         \mathrm{Q}_{1-\alpha}
        \left(
        \left\{
        \mathrm{UCB}_{t}(k) : 
        k \in H
        \right\}
        \right),
\end{equation}
and
\begin{equation}
\label{eq:quantilebound2}
     \ExpRewardInvCDF_{H}(1 - \alpha) - \frac{\Delta_j}{2} 
        <
         \mathrm{Q}_{1-\alpha}
        \left(
        \left\{
        \mathrm{LCB}_{t}(k) : 
        k \in H
        \right\}
        \right).
\end{equation}
\begin{proof}
To show \eqref{eq:quantilebound}, it is sufficient to show that for any arm $k \in H$, the following implication is true:
\begin{equation}
\label{eq:suff_quant_bound}
    \mu_k \le \ExpRewardInvCDF_{H}(1 - \alpha)  
    \implies
    \mathrm{UCB}_{t}(k) <
    \ExpRewardInvCDF_{H}(1 - \alpha) + \frac{\Delta_j}{2}.
\end{equation}
This is because there are at least a $(1-\alpha)$ fraction of arms in $H$ satisfying $\mu_k \le \ExpRewardInvCDF_{H}(1 - \alpha)$, and each such arm $k$ further satisfies $\mathrm{UCB}_{t}(k) <
    \ExpRewardInvCDF_{H}(1 - \alpha) + \frac{\Delta_j}{2}$ under \eqref{eq:suff_quant_bound},
which immediately gives \eqref{eq:quantilebound}.   
Likewise, for \eqref{eq:quantilebound2}, it is sufficient to show that each arm in the set
$\left\{
k \in H: \mu_k \ge \ExpRewardInvCDF_{H}(1 - \alpha) 
\right\}$  satisfies $\ExpRewardInvCDF_{H}(1 - \alpha) - \frac{\Delta_j}{2}  < \mathrm{LCB}_{t}(k)$.
 We give a proof for \eqref{eq:suff_quant_bound} which yields \eqref{eq:quantilebound}, and omit the similar details that yield \eqref{eq:quantilebound2}.

For each arm $k \in H$ satisfying $\mu_k \le \ExpRewardInvCDF_{H}(1 - \alpha)$, we have
by the definitions \eqref{def:gap} and \eqref{eq:delta_prime} that
\begin{equation}
\label{eq:Delta_K}
    \Delta_k = 
    \max\left\{ 
    \Delta, \Delta_H, \Delta_0, \ExpRewardInvCDF_{H}(1 - \alpha) - \mu_k
    \right\}.
\end{equation}
(Recall the integer indexing convention in \eqref{eq:shorthand}.)  Moreover, we have by the same definitions that
\begin{equation}
\label{eq:Delta_J}
     \Delta_j = 
    \max\left\{ 
    \Delta, \Delta_H, \Delta_0, |\mu_j - \ExpRewardInvCDF_{H}(1 - \alpha)|
    \right\}.
\end{equation}
We consider two cases: (i) $\Delta_j \ge \Delta_k$; and 
(ii) $\Delta_k > \Delta_j$. For the first case, we have
\begin{align*}
    \ExpRewardInvCDF_{H}(1 - \alpha) + \frac{\Delta_j}{2} 
    \ge \mu_k + \frac{\Delta_k}{2} 
    > \mathrm{UCB}_{t}(k),
\end{align*}
where the first inequality follows
from the assumptions on $\mu_k$ and $\Delta_k$,
and the second inequality follows from \eqref{eq:mu_j confidence bounds}.
For the second case, we must have $\Delta_k = \ExpRewardInvCDF_{H}(1 - \alpha) - \mu_k$ by \eqref{eq:Delta_K} and \eqref{eq:Delta_J}, and so
\begin{equation*}
    \ExpRewardInvCDF_{H}(1 - \alpha) + \frac{\Delta_j}{2} 
    = \mu_k + \Delta_k  + \frac{\Delta_j}{2} 
    \ge \mu_k + \Delta_k 
    > \mathrm{UCB}_{t}(k),
\end{equation*}
where the last inequality follows from \eqref{eq:mu_j confidence bounds}.
Combining the two cases gives us \eqref{eq:suff_quant_bound} as desired.
\end{proof}
\end{lemma}


\begin{proof}[Proof of Lemma~\ref{lem:stop_pull_condition}]
    Let arm $j \in \{1, \cdots, n\}$ be arbitrary, and let $H$ be
the group containing $j$.
If $j \not\in \gB_{t}$, then we also have 
$j \not\in \gB_{t+1}$, and we are done.
Therefore, we may assume without loss
of generality that $j \in \gB_{t}$.
Likewise, we assume that the while-loop of Algorithm~\ref{alg:elimination} has not terminated yet.
    We consider four cases for $\Delta_j$:
\begin{itemize}[topsep=0pt, itemsep=0pt]
    \item [(i)] $\Delta_j = \Delta'_j$
    \item [(ii)] $\Delta_j = \Delta$
    \item [(iii)] $\Delta_j = \Delta_H$
    \item [(iv)] $\Delta_j = \Delta_0 > \Delta_H$,
\end{itemize}
and show that in each case, at least one of the following happens: 
\begin{itemize}[topsep=0pt, itemsep=0pt]
    \item [(a)] $j \not\in \gB_{t+1}$, that is,
$j$ is eliminated at the end of round $t$; 
    \item [(b)] $|\gC_{t+1}| = 1$, that is,
the first condition of the while-loop termination is satisfied; 
    \item [(c)] $\Delta^{t+1} \le \Delta$, that is,
the second condition of the while-loop termination is satisfied.  
\end{itemize}

\paragraph{Case (i):}
$\Delta_j = \Delta'_j$.
By the definition of $\Delta'_j$ (see \eqref{eq:delta_prime}), arm $j$ is not a $(1-\alpha)$-quantile arm in group~$H$, which suggests that it should not be included in $m_{t+1}(H)$.
We will show that, indeed, arm~$j$
is no longer a potential $(1-\alpha)$-quantile arm in $H$ in round $t+1$, i.e.,
$j \not\in m_{t+1}(H)$, and so $j \not\in \gB_{t+1}$.

We consider two sub-cases: $ \mu_j >  \mu_{j_{\alpha}(H)} $ and 
$ \mu_j <  \mu_{j_{\alpha}(H)} $.
If  $ \mu_j >  \mu_{j_{\alpha}(H)} $, we have
    \begin{align}
        \mathrm{LCB}_{t}(j)
        & > \mu_j - \frac{\Delta_j}{2} \label{eq:case_1.1_1} \\
        & = \mu_j - \frac{\Delta_j'}{2} \label{eq:case_1.1_2} \\
        & = \left(\mu_{j_{\alpha}(H)} + \mu_j  - \mu_{j_{\alpha}(H)}\right)- \frac{\mu_j - \mu_{j_{\alpha}(H)}}{2} \label{eq:case_1.1_3} \\
        & =\mu_{j_{\alpha}(H)} + \frac{\mu_j - \mu_{j_{\alpha}(H)}}{2} \label{eq:case_1.1_4} \\
        & = \mu_{j_{\alpha}(H)} + \frac{\Delta_j'}{2} \label{eq:case_1.1_5} \\
        & \ge
         \mathrm{Q}_{1-\alpha}
        \left(
        \left\{
        \mathrm{UCB}_{t}(j) : 
        j \in H
        \right\}
        \right),
        \label{eq:violate_m_t(G)}
    \end{align}
    where 
    \eqref{eq:case_1.1_1} uses \eqref{eq:mu_j confidence bounds},
    \eqref{eq:violate_m_t(G)} uses \eqref{eq:quantilebound},
    and both \eqref{eq:case_1.1_3} and \eqref{eq:case_1.1_5} use 
    the definition of $\Delta'_j$.
    From \eqref{def:m_t} and \eqref{eq:violate_m_t(G)}, we have that 
    $j \not\in m_{t+1}(H)$, hence $ j \not\in \gB_{t+1}$. 
    The case for $ \mu_j <  \mu_{j_{\alpha}(H)} $ is similar, which leads to
    $\mathrm{UCB}_{t}(j) < 
     \mathrm{Q}_{1-\alpha}
    \left(
    \left\{
    \mathrm{LCB}_{t}(j') : 
    j' \in H
    \right\}
    \right),$
    and so $ j \not\in \gB_{t+1}$.

\paragraph{Case (ii):}
$\Delta_j = \Delta$. The idea of this case is to show that each remaining group~$H' \in \gC_t$ already satisfies \eqref{eq:optimal H},
meaning that the while-loop should be terminated.
Indeed, we have
\begin{equation}
\label{eq:case_1.2}
    \Delta^{(t+1)} = 2 U\bigg(t,  \frac{\delta}{n} \bigg) 
    =
    2 \ U\bigg(T_j(t),  \frac{\delta}{n} \bigg)
    < \frac{\Delta_j}{2}  = \frac{\Delta}{2} \le \Delta,
\end{equation}
where the first two equalities follow from Remark~\ref{rem: efficient implementation},
and the first inequality follows from \eqref{eq:mu_j confidence bounds}.
Therefore, the while-loop of Algorithm~\ref{alg:elimination} is terminated and arm $j$ is no longer pulled after round $t$.

\paragraph{Case (iii):}
$\Delta_j = \Delta_H$.
By the definition of $\Delta_H$ (see \eqref{eq:Delta_H}), group $H$ is not an optimal group, which suggests that group $H$ should not be included in $\gC_{t+1}$.
We will show that, indeed, $H \not\in \gC_{t+1}$, and so $j \not\in \gB_{t+1}$.

Let $H^* \in \argmax\limits_{H \in \gH} \ExpRewardInvCDF_{H}(1 - \alpha) $ be an optimal group, and $j_{\alpha}(H^*)$ be its
$(1-\alpha)$-quantile arm. By these definitions,
$\Delta_{j_{\alpha}(H^*)} = \max\left\{\Delta, \Delta_0\right\} \le \Delta_j = \Delta_H$.
Then, we have
     \begin{align}
        \mathrm{Q}_{1-\alpha}
        \left(
        \left\{
        \mathrm{LCB}_{t}(k) : 
        k \in H^*
        \right\}
        \right)  
        & >
        \mu_{j_{\alpha}(H^*)}
        - \frac{\max\left\{\Delta, \Delta_0\right\}}{2}
        \label{eq:case_2.2_1} \\
        & \ge \mu_{j_{\alpha}(H^*)} - \frac{\Delta_H}{2}
        \label{eq:case_2.2_2} \\
        & = \mu_{j_{\alpha}(H)} + \frac{\Delta_H}{2} 
        \label{eq:case_2.2_3} \\
        & = \mu_{j_{\alpha}(H)} + \frac{\Delta_j}{2} 
        \label{eq:case_2.2_9} \\
        & \ge
         \mathrm{Q}_{1-\alpha}
        \left(
        \left\{
        \mathrm{UCB}_{t}(j) : 
        j \in H
        \right\}
        \right),
        \label{eq:case_2.2_10}
    \end{align}
    where 
    \eqref{eq:case_2.2_1} uses \eqref{eq:quantilebound2} with $H^*$ and $j_{\alpha}(H^*)$; 
    \eqref{eq:case_2.2_2} uses the assumption that $\Delta_H = \Delta_j \ge  \max\left\{\Delta, \Delta_0 \right\}$;
     \eqref{eq:case_2.2_3} uses the definition of $\Delta_H$; and
    \eqref{eq:case_2.2_10} uses \eqref{eq:quantilebound} with $H$ and $j$;
    Then, \eqref{def:C_t} and \eqref{eq:case_2.2_10} imply that  $H \not\in \gC_{t+1}$ as desired.

\paragraph{Case (iv):}
$\Delta_j = \Delta_0 > \Delta_H$. 
By the definition of $\Delta_0$ (see \eqref{eq:Delta_0}), group $H$ must be the unique optimal group, which suggests that $\gC_{t+1} =\{H\}$.
We will show that, indeed, all other groups $H'$, which are non-optimal, are not included in $\gC_{t+1}$, i.e., $\gC_{t+1} = \{H\}$.
Then the while-loop of Algorithm~\ref{alg:elimination} is terminated because the criterion $|{\gC}_{t+1}| > 1$ no longer holds.

Let $H'$ be an arbitrary suboptimal group. Then
\begin{equation}
\label{eq:delta_nonoptimalgroup}
    \Delta_{H'} = \mu_{j_{\alpha}(H)} - \mu_{j_{\alpha}(H')} \ge \Delta_0 >  0.
\end{equation}
Furthermore,
its $(1-\alpha)$-quantile arm $j_{\alpha}(H')$ has gap 
$\Delta_{\mu_{j_{\alpha}(H')}} = 
\max\{\Delta_{H'}, \Delta, \Delta_0\} = \Delta_{H'}$ (note that $\Delta \le \Delta_0$ in the current case). Then, we have
    \begin{align}
        \mathrm{Q}_{1-\alpha}
        \left(
        \left\{
        \mathrm{LCB}_{t}(j) : 
        j \in H
        \right\}
        \right)  
        & > \mu_{j_{\alpha}(H)} - \frac{\Delta_0}{2} \label{eq:case_4.1} \\
        & \geq \mu_{j_{\alpha}(H)} - \frac{\mu_{j_{\alpha}(H)} - \mu_{j_{\alpha}(H')} }{2} \label{eq:case_4.2} \\
        & = \mu_{j_{\alpha}(H')} + \frac{\mu_{j_{\alpha}(H)} - \mu_{j_{\alpha}(H')}}{2} 
        \label{eq:case_4.3} \\
        & = \mu_{j_{\alpha}(H')} + \frac{\Delta_{H'}}{2} \label{eq:case_4.4} \\
        & \ge
         \mathrm{Q}_{1-\alpha}
        \left(
        \left\{
        \mathrm{UCB}_{t}(j') : 
        j' \in H'
        \right\}
        \right),
        \label{eq:case_4.5}
    \end{align}
    where  
    \eqref{eq:case_4.1} uses \eqref{eq:quantilebound2};
    \eqref{eq:case_4.2} and \eqref{eq:case_4.4} use \eqref{eq:delta_nonoptimalgroup};
    and 
    \eqref{eq:case_4.5} uses \eqref{eq:quantilebound}
    with $H'$ and $j_{\alpha}(H')$.
    Then, \eqref{def:C_t} and \eqref{eq:case_4.5} imply that 
     $H' \not\in \gC_{t+1}$ for all suboptimal groups $H'$.
\end{proof}

%% file: appendix_lb.tex

\section{Proof of Theorem \ref{thm:lb} (Worst-Case Lower Bound)} 
\label{app:pf_lower}

We consider $\alpha = \frac{1}{2}$, so that the problem is one of best median identification; other values of $\alpha \in (0,1)$ can be handled with only minor changes.  It is useful to first handle the case of two groups before handling the general case.

\subsection{Proof for the Two-Group Case}

For $|\mathcal{G}|=2$, we construct two instances, one with group 1 being optimal and one with group 2 being optimal, and with the suboptimal group failing to satisfy \eqref{eq:best_group_relax_eps_Delta} in both cases.  Then, we will show that the two instances are hard to distinguish given the rewards unless $T \ge \Omega\big( \frac{1}{ \Delta^2 \epsilon^2 } \big)$.  

\textbf{Reduction to a binary decision problem.} We specialize to a simple setting in which there are two types of arms (``good'' and ``bad''); this specialization is analogous to how the binary-valued problem in \cite{jamieson2016power} is a special case of the quantile-based problem with general reservoir distributions in \cite{aziz2018pure}.\footnote{The difference in the two settings is highlighted by the $\frac{1}{\Delta^2\alpha}$ scaling in \cite{jamieson2016power}, compared to $\frac{1}{\Delta^2 \epsilon^2}$ in our setting.}  
Suppose that the arm distributions are Bernoulli, and that group~2's arms have two possible means, $q = \frac{1+\Delta}{2}$ and $\bar{q} = 1-q = \frac{1-\Delta}{2}$.  We call these \emph{good arms} and \emph{bad arms} respectively.  Then, the two instances are defined as follows:
\begin{itemize} \itemsep0ex
    \item (Good instance) In group 2, the reservoir distribution places probability mass $p = \frac{1+\epsilon}{2}$ on the good arms, and mass $\bar{p} = 1-p = \frac{1-\epsilon}{2}$ on bad arms.
    \item (Bad instance) In group 2, the reservoir distribution places probability mass $p$ on the bad arms, and $\bar{p}$ on the good arms.
\end{itemize}
In both instances, in group 1, the reservoir distribution places probability mass 1 on arms with mean~$\frac{1}{2}$.  Hence, group 2 has the higher median in the good instance, and the lower median in the bad instance.  
Observe that in these instances, up to constant rescaling of $\epsilon$ and $\Delta$ (e.g., changing them by a factor of $4$), we find that the suboptimal group fails to satisfy \eqref{eq:best_group_relax_eps_Delta}.  Since the theorem ignores constant factors, we can ignore this scaling and proceed with $\epsilon$ and $\Delta$ as above.

Now, if the bandit algorithm knows that the arms and rewards are produced by one of the two instances, the problem is simply reduced to determining whether the instance is good or bad using arms from group 2 alone.  (Arms from group 1 provide no information for distinguishing the instances.)  Accordingly, all arms mentioned in the subsequent analysis are those from group 2.

In the following, we assume that the bandit algorithm is deterministic.  We will derive a lower bound that holds when the good and bad instances each occur with probability $\frac{1}{2}$, and the same lower bound then holds for randomized algorithms by Yao's minimax principle.  In addition, it suffices to prove \eqref{eq:lb_worst} when the target error probability $\delta \in \big(0,\frac{1}{2}\big)$ is an arbitrary fixed positive value (e.g., $0.49$).  This is because a success probability of $0.51$ (say) could be amplified to any value in $\big(\frac{1}{2},1\big)$ by independently repeating the algorithm a constant number of times and taking a majority vote, amounting to only a constant factor increase in the number of arm pulls.  Finally, we may assume that the time horizon $T$ is fixed, because attaining $\EE[T] \le T^*$ implies attaining $T \le CT^*$ with probability at least $1-\frac{1}{C}$ by Markov's inequality, where $C$ can be arbitrarily large.

\textbf{Analysis of a single arm.}  Consider an arbitrary time step of the algorithm and an arbitrary arm that has been pulled some number of times to produce a vector $\Xv$ of i.i.d.~rewards.  We first study the likelihood ratio
\begin{equation}
    L(\Xv) = \frac{ \PP_2[\Xv] }{ \PP_1[\Xv] },
\end{equation}
where $\PP_2[\cdot]$ (respectively, $\PP_1$) denotes probability under the good (respectively, bad) instance.  (We will similarly use the notation $\EE_1[\cdot]$ and $\EE_2[\cdot]$ later.)  By our assumption of independent Bernoulli arms, a direct calculation gives that when 
\begin{equation}
    (\text{\#1s in $\Xv$}) - (\text{\#0s in $\Xv$}) = d, \label{eq:diff_d}    
\end{equation}
the likelihood ratio $L(\Xv)$ is given by the following function:
\begin{equation}
    f(d) = \frac{pq^d + \p\q^d}{\p q^d+p\q^d}.
\end{equation}
For brevity, we refer to $d$ as the \emph{score}.

The following lemmas are central to our analysis, and together they show that the next average of $f(\cdot)$ behaves similarly under the two instances (conditioned on the history so far).  Since the proofs are rather technical, they are deferred to Appendices \ref{sec:lb_pf_bad} and \ref{sec:lb_pf_good}.

\begin{lemma} \label{lem:score_bad}
    Under the bad instance, for a given arm, let $\Hc^{(d)}$ be any history of arm pulls giving a score $d$ in \eqref{eq:diff_d}, and let $d'$ be a random variable indicating the updated score following one additional arm pull after $\Hc^{(d)}$.  Then, we have
    \begin{equation}
        \EE_1[ f(d') | \Hc^{(d)} ] = f(d).
    \end{equation}
\end{lemma}

\begin{lemma} \label{lem:score_good}
    Under the good instance, for a given arm, let $\Hc^{(d)}$ be any history of arm pulls giving a score $d$ in \eqref{eq:diff_d}, and let $d'$ be a random variable indicating the updated score following one additional arm pull after $\Hc^{(d)}$.  Then, we have
    \begin{equation}
        \EE_2[ f(d') | \Hc^{(d)} ] = f(d) + O(\Delta^2\epsilon^2). \label{eq:good_gen}
    \end{equation}
\end{lemma}

\textbf{Analysis of all arms.} Suppose that after some number $t$ of arm pulls, a history $\Hc_t$ has been observed consisting of $N$ arms that have been pulled, with associated scores $d_1,\ldots, d_N$.  Based on the history, the algorithm selects an arm $i^*$, leading to (randomly) updated scores $d'_1,\ldots, d'_N$ (with $d_j = d'_j$ for all $j\neq i^*$).  We claim that
\begin{equation}
    \EE_2\Big[ \prod_{i=1}^N f(d'_i) \,\Big|\, \Hc_t \Big] \leq (1+\mathcal{O}(\epsilon^2 \Delta^2)) \prod_{i=1}^N f(d_i).
    \label{eq:choose_next_arm}
\end{equation}
To see this, we note that $i^*$ is deterministic given $\Hc_t$.  Hence, we obtain \eqref{eq:choose_next_arm} from Lemma \ref{lem:score_good}; the values of $d_j$ for $j \ne i^*$ do not change, so we can factorize their terms out on both sides in \eqref{eq:choose_next_arm}.  Note also that since the arms are independent and the (good) instance is fixed, other arms' rewards do not impact the next reward of arm $i^*$.  

Another point worth mentioning is that a new arm may also be chosen at time $t$.  This is easily captured by the above definitions by increasing $N$ by one and then setting $d_N = 0$ (the score of any not-yet-pulled arm is trivially zero).  In fact, without loss of generality, we may assume that $T$ arms are requested at the very start, and accordingly set $N = T$ (though not all of these arms would end up being pulled).


Hence, letting $L_t = \prod_i f(d_i)$ at time $t$, we have $L_0=1$, and if $\Hc_t$ is the history up to time $t$, then
\begin{equation}
    \EE_2[L_{t+1} | \Hc_t] \leq (1+\mathcal{O}(\epsilon^2 \Delta^2)) L_t.
\end{equation}
Applying this recursively, we have after $T$ arm pulls that
\begin{equation}
    \EE_2[L_T] \leq (1+\mathcal{O}(\epsilon^2 \Delta^2))^T \leq \exp(\mathcal{O}(T\epsilon^2 \Delta^2)).
\end{equation}
Hence, if $T = O\big( \frac{1}{\epsilon^2\Delta^2}\big)$ with a small enough implied constant, then we have $\EE_2[L_T] \le 1.5$.  Applying Markov's inequality, it follows that $\PP_2[L_T \ge 2] \le \frac{3}{4}$.  However, when $L_T < 2$, the probability under $\PP_1$ of observing the same history is at least half of the corresponding probability under $\PP_2$.  Given that history, the algorithm can only be correct under at most one of $\PP_1$ and $\PP_2$.  Hence, the failure probability (in selecting between the good and bad instances) is $\Omega(1)$ when $T = O\big( \frac{1}{\epsilon^2\Delta^2}\big)$. In view of our preceding reduction to a binary decision problem, this gives the desired result stated in Theorem~\ref{thm:lb} for the case that $|\mathcal{G}|=2$.

\subsection{Proof for the General Case}

The above analysis of the two-group case can be summarized as showing the following: Let group~1 consist entirely of arms with mean $\frac{1}{2}$, and consider two instances where (i) group 2 contains a fraction $\frac{1-\epsilon}{2}$ of good arms, and (ii) group 2 contains a fraction $\frac{1+\epsilon}{2}$ of good arms.  Then any algorithm identifying the optimal group on both of these instances with constant probability (better than random guessing) must have $\EE_2[T] \ge \Omega\big( \frac{1}{\Delta^2\epsilon^2} \big)$, where the subscript to $\EE$ indicates the instance.  Moreover, since the problem is symmetric with respect to the two instances (namely, swapping 0 rewards with 1 rewards simply amounts to replacing ${\rm Bernoulli}(p)$ arms by ${\rm Bernoulli}(1-p)$), such an algorithm must also have $\EE_1[T] \ge \Omega\big( \frac{1}{\Delta^2\epsilon^2} \big)$.\footnote{Viewed differently, we can simply assume without loss of generality that $\EE_1[T]$ and $\EE_2[T]$ coincide to within a constant factor.  To see this, note that if $\EE_1[T]$ were much larger than $\EE_2[T]$ (say), we could simply run the algorithm up to time $C \EE_2[T]$ for some large $C$, and guess that we are in instance 1 if the algorithm has not terminated.  This would increase the error probability under instance 2 by at most $\frac{1}{C}$, which is arbitrarily small.}

We now turn to the case of a more general number of groups.  We define good and bad arms in the same way as the two-group case (i.e., with means $\frac{1+\Delta}{2}$ and $\frac{1-\Delta}{2}$), but we now generalize the structure of the groups themselves.  As before, we let group 1's reservoir distribution place probability mass~1 on arms with mean~$\frac{1}{2}$.  We then define $|\mathcal{G}|$ different bandit instances according to the proportion of good vs.~bad arms in the remaining groups.  Specifically, letting $p_i$ denote the proportion of good arms in group $i \in \{2,\dotsc,|\mathcal{G}|\}$, we define the following:
\begin{itemize} \itemsep0ex
    \item In instance 1, we let $p_i = \frac{1-\epsilon}{2}$ for all $i \in \{2,\dotsc,|\mathcal{G}|\}$.
    \item In instance $j$ for $j \in \{2,\dotsc,|\mathcal{G}|\}$, we let $p_j = \frac{1+\epsilon}{2}$, and we let $p_i = \frac{1-\epsilon}{2}$ for all $i \in \{2,\dotsc,|\mathcal{G}|\} \setminus \{j\}$.
\end{itemize}
Thus, we clearly have for all $j \in \{1,\dotsc,|\mathcal{G}|\}$ that group $j$ is the optimal group in instance $j$.  Moreover, up to rescaling of $\epsilon$ and $\Delta$, it is the only group that satisfies \eqref{eq:best_group_relax_eps_Delta}.  Thus, under instance $j$, the algorithm's estimate $\hat{j}$ of the optimal group must equal $j$ with probability at least $1-\delta$ for some $\delta \in \big(0,\frac{1}{2}\big)$ in accordance with the statement of Theorem \ref{thm:lb}.

Let $\PP_j$ and $\EE_j$ denote probability and expectation under instance $j$.  Moreover, for each group $i$, let $T_i$ be a random variable indicating the total number of pulls of arms from group $i$.  Since $\sum_{i=1}^{|\mathcal{G}|} T_j = T$ almost surely, there must exist a group $j^* \in \{2,\dotsc,|\mathcal{G}|\}$ such that $\EE_{1}[T_{j^*}] \le \frac{ \EE_1[ T ] }{ |\mathcal{G}|-1 }$.  We proceed with a reduction to the two-group setting with group 1 and group $j^*$.

This reduction is based on the observation that all groups except $j^*$ are identical under instance 1 and instance $j^*$ (including group 1), so their arm pulls have no power in distinguishing the two instances.  Thus, an algorithm for the $|\mathcal{G}|$-group succeeding on these two instances immediately implies that there exists an algorithm succeeding for the 2-group setting (with only groups $1$ and $j^*$) using an average number of arm pulls given by the quantity $\EE_{1}[T_{j^*}]$ mentioned above.  Then, the result stated in the first paragraph of this subsection implies that $\EE_{1}[T_{j^*}]\ge \Omega\big( \frac{1}{\Delta^2\epsilon^2} \big)$, and combining this with $\EE_{1}[T_{j^*}] \le \frac{ \EE_1[ T ] }{ |\mathcal{G}|-1 }$ and $|\mathcal{G}| \ge 2$ completes the proof of Theorem \ref{thm:lb}.

\subsection{Proof of Lemma \ref{lem:score_bad} (Bad Instance)} \label{sec:lb_pf_bad}

By a simple application of Bayes' rule, under the bad instance, conditioned on $\Hc^{(d)}$ with score $d$, the probability that the arm under consideration is good is
\begin{equation}
    \frac{\p q^d}{\p q^d + p \q^d}. \label{eq:p_good}
\end{equation}
As a sanity check, setting $d=0$ makes this quantity equal to the prior, $\p$.

From \eqref{eq:p_good} and the fact that good (resp., bad) arms are Bernoulli with mean $q$ (resp., $\q$), the probability of the next draw returning 1 is given by
\begin{equation}
    q \frac{\p q^d}{\p q^d + p \q^d} + \q  \frac{p \q^d}{\p q^d + \q^d}.
\end{equation}
Therefore, we have
\begin{align}
    \EE_1[f(d')|\Hc^{(d)}] 
    &= \left( q \frac{\p q^d}{\p q^d + p \q^d} + \q  \frac{p \q^d}{\p q^d + p \q^d}\right) f(d+1) + \left(\q \frac{\p q^d}{\p q^d + p \q^d} + q  \frac{p \q^d}{\p q^d + p \q^d}\right) f(d-1) \\
    &= \frac{\p q^{d+1} + p \q^{d+1}}{\p q^d + p \q^d}  \frac{pq^{d+1}+\p\q^{d+1}}{\p q^{d+1}+p\q^{d+1}} + \frac{q\q(\p q^{d-1} + p\q^{d-1})}{\p q^d + p \q^d} \frac{pq^{d-1}+\p\q^{d-1}}{\p q^{d-1}+p\q^{d-1}} \\
    &= \frac{pq^{d+1} + \p \q^{d+1} + q\q(pq^{d-1} + \p \q^{d-1}) }{\p q^d + p \q^d}\\
    &= \frac{(q+\q)(pq^{d}+\p\q^d)}{\p q^d + p \q^d}\\
    &= f(d).
\end{align}

\subsection{Proof of Lemma \ref{lem:score_good} (Good Instance)} \label{sec:lb_pf_good}

We use similar ideas to the proof of Lemma \ref{lem:score_bad}.  This time, the conditional probability (given $\mathcal{H}^{(d)}$) of having a good arm is
\begin{equation}
    \frac{pq^d}{pq^d+\p\q^d}, \label{eq:p_bad}
\end{equation}
since in the good instance we swap $p$ and $\p$.  Observe that the difference in the two conditional probabilities (\eqref{eq:p_good} and \eqref{eq:p_bad}) can be upper bounded as follows:
\begin{equation}
    \frac{pq^d}{pq^d+\p\q^d} - \frac{\p q^d}{\p q^d + p \q^d} = \frac{q^d\q^d (p^2-\p^2)}{(pq^d + \p\q^d)(\p q^d + p \q^d)} \leq \frac{q^d\q^d (p^2-\p^2)}{(pq^d)(p\q^d)} \leq 4(p^2 - \p^2) = 4\epsilon, \label{eq:eps_bound}
\end{equation}
where the factor of 4 comes from $p \ge \frac{1}{2}$, and the last step substitutes $p = \frac{1+\epsilon}{2}$.

Now let $A$ denote the event that the next arm pull returns 1, and let $B$ denote the event that the arm under consideration is good, and observe that
\begin{align}
    &\PP_2[A | \Hc^{(d)}] - \PP_1[A | \Hc^{(d)}] \\
     & = q \cdot  (\PP_2[B |\Hc^{(d)}] - \PP_1[B | \Hc^{(d)}]) +\q \cdot  (\PP_2[B^c |\Hc^{(d)}] - \PP_1[B^c | \Hc^{(d)}]) \\
    & = (q-\q) \left(\frac{pq^d}{pq^d+\p\q^d} - \frac{\p q^d}{\p q^d + p \q^d}\right)\\
    & \leq 4\epsilon\Delta,
\end{align}
where we substituted \eqref{eq:eps_bound} and $q = \frac{1+\Delta}{2}$. 
It follows that
\begin{align}
     &\EE_2[f(d')|\Hc^{(d)}]  - \EE_1[f(d')|\Hc^{(d)}]  \nonumber \\
     & = f(d+1) \big(\PP_2[A|\Hc^{(d)}] - \PP_1[A|\Hc^{(d)}]\big) + f(d-1) \big(\PP_2[A^c|\Hc^{(d)}] - \PP_1[A^c|\Hc^{(d)}]\big)\big) \\
     & \leq (f(d+1)-f(d-1)) \cdot 4\epsilon\Delta. \label{eq:generator_dependence}
\end{align}
To bound $f(d+1)-f(d-1)$, consider the function
\begin{equation}
    g(x) = \frac{pe^x+\p}{\p e^x+p},
\end{equation}
and note that
\begin{equation}
    f(d) = \frac{pq^d + \p \q^d}{\p q^d + p\q^d} = g(d \ln (q/\q)).
\end{equation}
Taking the derivative, we obtain
\begin{equation}
    \frac{d}{dx} g'(x) = \frac{e^x(p^2-\p^2)}{(\p e^x+p)^2} \leq \frac{e^x(p^2-\p^2)}{2p\p e^x} = \frac{p^2-\p^2}{2p\p} = \mathcal{O}(\epsilon),
\end{equation}
where the final step uses $p = \frac{1+\epsilon}{2}$ and the assumption $\epsilon \in \big(0,\frac{1}{4}\big)$.
Hence, we can bound the difference of interest using a first-order Taylor expansion:
\begin{align}
    f(d+1)-f(d-1) 
    &= g((d+1)\ln(q/\q)) - g((d-1)\ln(q/\q)) \\
    &\leq 2\ln(q/\q) \cdot \sup_{x} g'(x) \\
    &\leq 2 \ln(q/\q) \cdot \mathcal{O}(\epsilon) = \mathcal{O}(\epsilon \cdot \Delta),
\end{align}
since $q = \frac{1+\Delta}{2}$ and $\Delta \in \big(0,\frac{1}{4}\big)$.

Combining this with (\ref{eq:generator_dependence}), and using Lemma \ref{lem:score_bad} for the $\EE_1$ term, we obtain (\ref{eq:good_gen}) as desired.

\section{Comparisons and Connection to Previous Works} \label{sec:comparisons}

\subsection{Brief Overview of Two Related Works}

Regarding \cite{wang2021max} (max-min grouped bandits with finitely-many arms) and \cite{aziz2018pure} (identifying a single $(1-\alpha)$-quantile arm in a non-grouped infinite-arm setting), some similarities and differences to our work are highlighted as follows:
\begin{itemize}
    \item At a high level, we adopt a two-step approach from \cite{aziz2018pure} of requesting a fixed number of arms (per group) from the reservoir distribution and then running a finite-arm algorithm.  Our setting gives rise to fundamental differences (see Appendix~\ref{sec:compare_aziz}), leading to us giving an improved upper bound via a multi-step approach (Section \ref{sec:improve}), as well as requiring distinct techniques for the lower bound (Section \ref{sec:lower}).
    \item As we mentioned in Section \ref{sec:num_pulls}, our finite-arm subroutine and its analysis are generally similar to the main algorithm in \cite{wang2021max} but with different details (see Appendices~\ref{sec:subroutine_details} and \ref{sec:appendix_finite_alg}).  The key distinction is not these details, but rather the challenges of incorporating the finite-arm algorithm into an infinite-arm framework.
\end{itemize}

\subsection{Suboptimality of General Structured Bandit Framework} \label{sec:structured}

As mentioned in Section \ref{sec:related}, various general frameworks for structured bandits have been introduced in the existing literature.  Perhaps most notably, our \emph{finite-arm} sub-problem (Appendix \ref{sec:subroutine_details}) can be viewed as a special case of best-arm identification in structured bandits, which was studied in \cite{huang2017structured}.  However, here we discuss how the general-purpose upper bound in \cite{huang2017structured} can be worse than our upper bound in Theorem \ref{thm:ub_se} (also written in continuous-index notation in Theorem \ref{thm:samplebound}).

When specialized to our setting, \cite[Thm.~10]{huang2017structured} gives gaps of the form $\max\{|\mu_j - c|,\Delta_0/2\}$, with $\Delta_0$ in \eqref{eq:Delta_0} and $c$ being the mid-point between the best two $(1-\alpha)$-quantiles.  We proceed by assuming that $\Delta_0 > 0$, as this follows from a uniqueness assumption made in \cite{huang2017structured}.

The weakness is that in suboptimal groups we could have $|\mu_j - c| = 0$ (or $\approx 0$) even when $\Delta'_{j}$ in~\eqref{eq:delta_prime} is large, since $\mu_j$ could be far above its own group's $(1-\alpha)$-quantile.  Moreover, this can happen no matter how small $\Delta_0$ is, if the group under consideration is not among the best two.  On the other hand, for arms in the top two groups and arms \emph{below the $(1-\alpha)$-quantile} in suboptimal groups, it can be checked that our gaps in~\eqref{def:gap} match those of \cite{huang2017structured} to within constant factors.  Hence, the gaps essentially match for many arms, but not for all arms.

More generally, we are unaware of any follow-up works to \cite{huang2017structured} that could be specialized to obtain our Theorem \ref{thm:ub_se}.  (We also emphasize that our main results are those for the infinite-arm setting.)

\subsection{More Detailed Comparison to Quantile-Based Good Arm Identification} \label{sec:compare_aziz}

The following discussion concerns our number of arm pulls stated in Corollary \ref{cor:samplebound_independent_H}.  

In \cite{aziz2018pure}, a related problem was studied in which there is only one infinite-arm group, and the goal is to identify a single arm in the top-$\alpha$ fraction of arms for some $\alpha \in (0,1)$.  Their final sample complexity bears some resemblance to ours, e.g., utilizing gaps that depend on quantities $b_1,b_2,\dotsc$ amounting to a similar partitioning to that in Section \ref{sec:num_pulls_without}.  In addition, their algorithm uses a similar two-step procedure to our Algorithm \ref{alg:main}, first acquiring ``sufficiently many'' arms and then running a finite-arm algorithm.

Notably, the partitioning in \cite{aziz2018pure} only requires $m = \frac{1}{\alpha}$, whereas we use $m = \frac{1}{\epsilon}$.  Similarly, our term $3\epsilon N$ in \eqref{eq:armpullbound_independent_H} leads to a further multiplicative $\frac{1}{\epsilon}$ factor, whereas no analogous term is present in \cite{aziz2018pure}.  
The main reason for these differences is that our problem formulation necessitates stricter conditions on the $N$ sampled arms (per group) in order for the finite-arm algorithm to lead to the desired goal.  In particular, in \cite{aziz2018pure} it is only required that at least one arm in the top-$\alpha$ fraction is chosen, whereas in our setting it is required that each group's sample quantile is $\epsilon$-close to the true quantile (see Event $A$ in \eqref{eq:sampled_quantile_sandwiched}).  The former is obtained with $N$ having $\frac{1}{\alpha}$ dependence, whereas the latter requires $\frac{1}{\epsilon^2}$ dependence, making our $3\epsilon N$ term in~\eqref{eq:armpullbound_independent_H} much more significant.

The preceding discussion also partially addresses the reason why our $m$ value is higher than that of \cite{aziz2018pure}, though in Section \ref{sec:improve}, we discuss how this could be at least partially alleviated by moving from a two-step algorithm to a multi-step algorithm.  Furthermore, in Section \ref{sec:lower}, we show that our upper bound is tight up to logarithmic factors in a ``worst-case'' sense (roughly amounting to all gaps equaling a common value $\Delta$) when $|\mathcal{G}|=2$ and $\delta = \Theta(1)$, thus justifying our larger choice of~$N$ compared to \cite{aziz2018pure}.

\section{Discussion on Instance-Dependent Lower Bounds} \label{sec:instance_lb}

In this appendix, we provide some discussion on instance-dependent lower bounds; these are absent in Section \ref{sec:lower}, where we provided a worst-case lower bound.

{\bf General-purpose change-of-measure techniques.} 
Change-of-measure techniques have proved to be very useful for obtaining instance-dependent lower bounds in diverse bandit problems \cite{kaufmann2016complexity}.  In our setting, however, it appears to be difficult to use existing general-purpose techniques while maintaining the correct $\epsilon$ dependence, as we will discuss further below.   

Despite this limitation, it is interesting to observe what dependencies on certain gaps could be obtained via this approach.  For simplicity, suppose that there are two groups.  Let $\mu_{j_{\alpha}(G_1)}$ and $\mu_{j_{\alpha}(G_2)}$ be $(1-\alpha)$-quantiles of the two groups, and suppose without loss of generality that $\mu_{j_{\alpha}(G_1)} > \mu_{j_{\alpha}(G_2)}$.  Then, let $\epsilon_0 \in [0,\alpha]$ be defined such that a fraction $\alpha-\epsilon_0$ of group 2's arms have mean exceeding $\mu_{j_{\alpha}(G_1)}$.  That is, a fraction $\epsilon_0$ of the arms have means between $\mu_{j_{\alpha}(G_2)}$ and $\mu_{j_{\alpha}(G_1)}$.  See Figure \ref{fig:eps} for an illustration (the quantity $\epsilon_k$ therein is discussed later).

\begin{figure}
    \begin{centering} 
        \includegraphics[width=0.45\columnwidth]{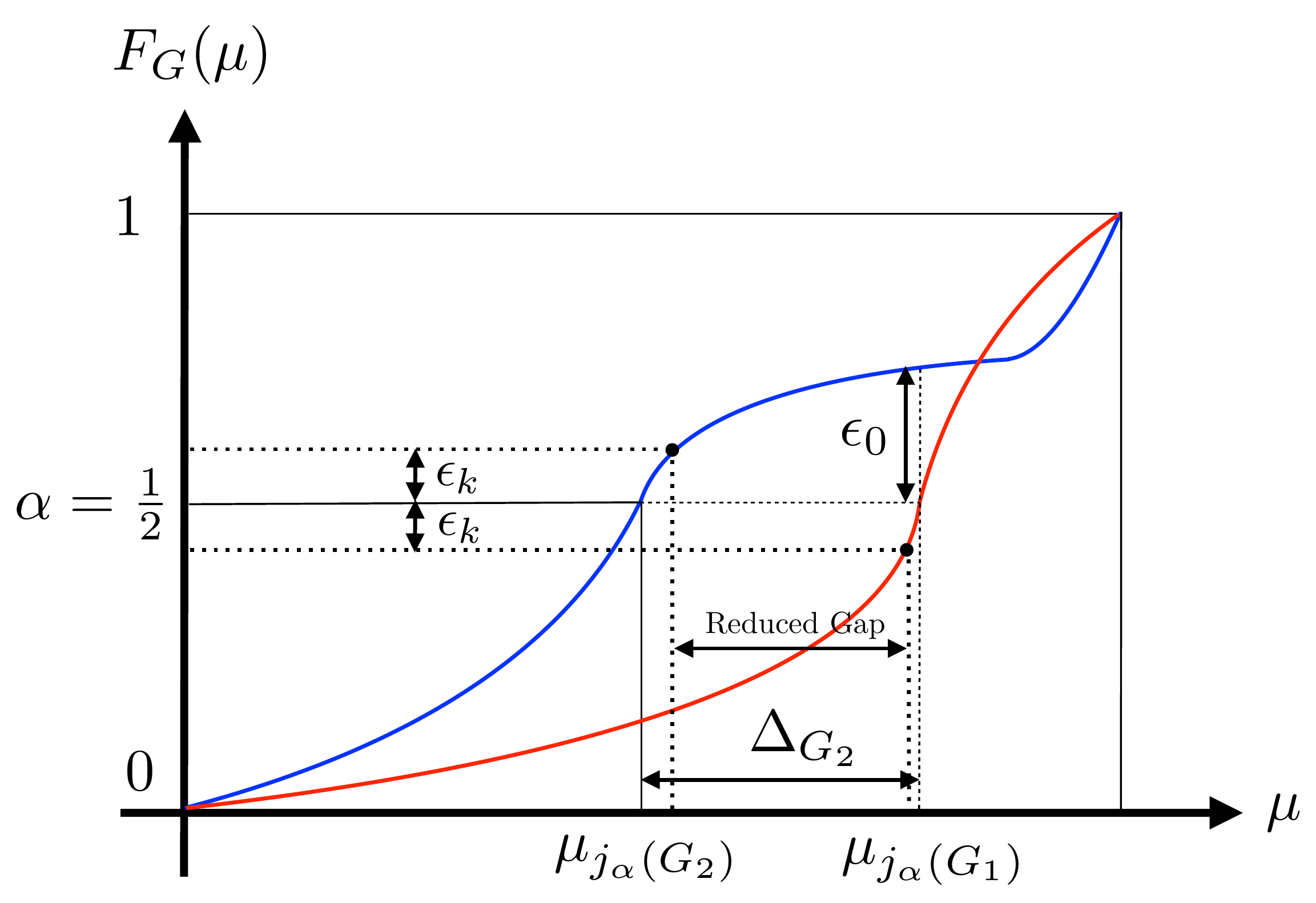} \quad
        \includegraphics[width=0.45\columnwidth]{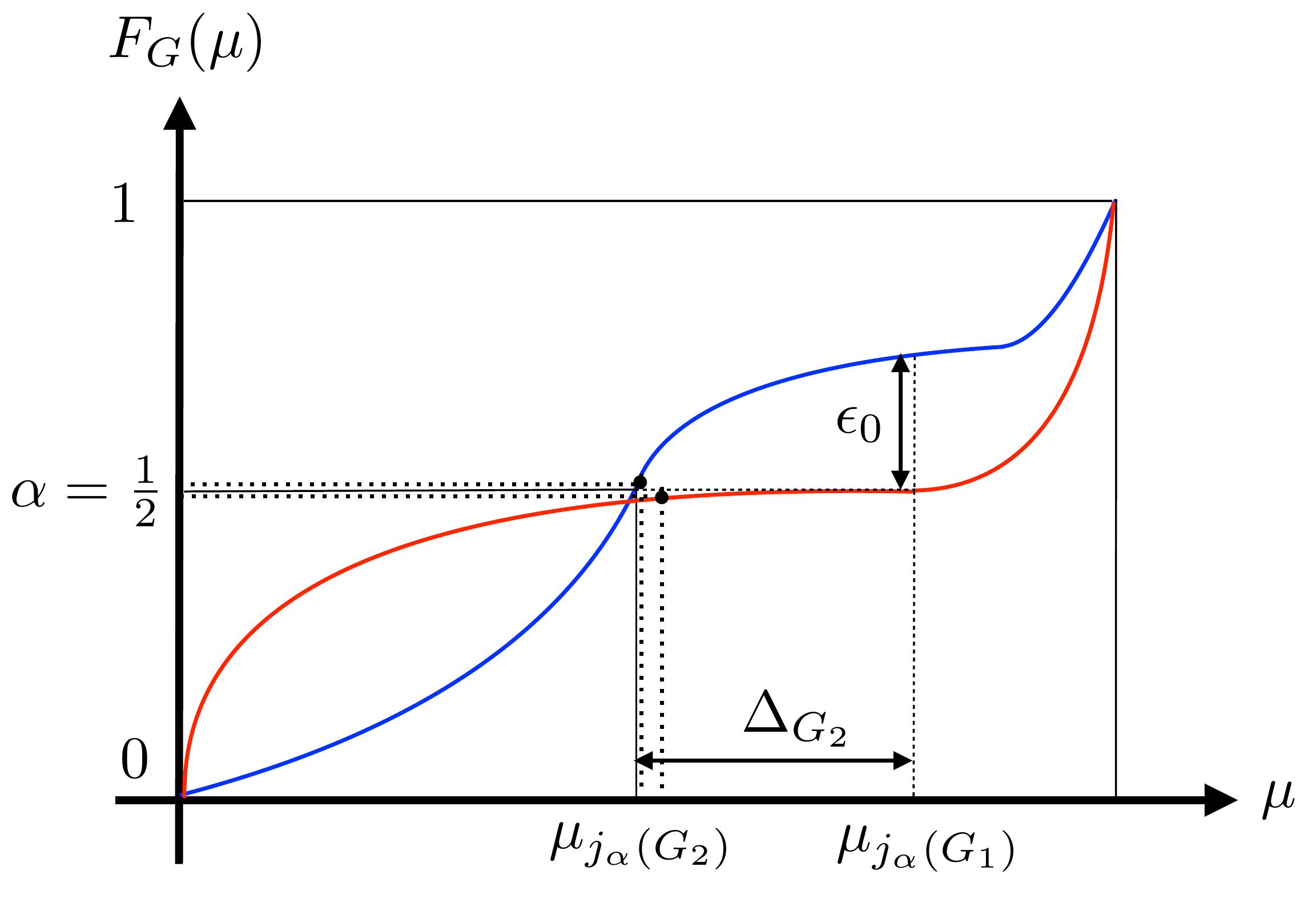}
        \par
    \end{centering}
    
    \caption{Examples of group reservoir distributions and the associated parameters.  In both cases, group 1 is the optimal group, and $\epsilon_0$ takes the same value.  In the left figure, the value $\epsilon_k$ associated with group 2 is similar to $\epsilon_0$, whereas in the right figure we have $\epsilon_k \ll \epsilon_0$. \label{fig:eps}}
\end{figure}

In the following, we consider the case $\Delta = 0$ for simplicity.
Observe that if any fraction exceeding $\epsilon_0$ of group 2's arms were ``shifted'' to just above $\mu_{j_{\alpha}(G_1)}$ (and they were all originally below this value), then group 2 would become the best one.  The standard change-of-measure technique in \cite{kaufmann2016complexity} then indicates that arms in this subset must be pulled a minimal number of times with dependence $\frac{1}{(\mu_{j_{\alpha}(G_1)} - \mu_{\min})^2}$, where $\mu_{\min}$ is the smallest mean in the partition.  One can then form roughly $\frac{1-\alpha + \epsilon_0}{\epsilon_0}$ disjoint partitions (up to rounding) capturing a fraction (marginally above) $\epsilon_0$ of group 2's reservoir distribution each, and apply the same reasoning.

The preceding discussion amounts to having gaps of the form $\mu_{j_{\alpha}(G_1)} - \mu_{\min}$.  This is consistent with our use of $\max\{\Delta_H,\Delta'_{H,j}\} = \Theta(\Delta_H + \Delta'_{H,j} )$ in our upper bound (though in the upper bound $\mu_{\max}$ would replace $\mu_{\min}$, as we discuss below).  Indeed, when $j$ indexes an arm that is in a suboptimal group and not in its top $(1-\alpha)$-quantile, $\Delta_H + \Delta'_{H,j}$ is precisely the difference between $\mu_{H,j}$ and the best group's $(1-\alpha)$-quantile.  Similar reasoning can be applied for the number of arm pulls in~$G^*$ itself (with a suitably modified choice of $\epsilon_0$).

{\bf Comparison of gap terms in upper and lower bounds.} In the preceding discussion, the value of $\epsilon_0$ is group-dependent, and it is interesting to compare this value to the quantities $\{\epsilon_k\}_{k=1}^K$ introduced in Section \ref{sec:improve}.  As we show in Figure \ref{fig:eps}, a value of $\epsilon_k$ similar in size to $\epsilon_0$ is often large enough to retain a large gap (left example), but in the worst case the required $\epsilon_k$ may be significantly smaller than $\epsilon_0$ (right example).  The separation between the two is very much dependent on the shapes of the underlying reservoir distributions.

Apart from this key difference, the upper and lower bounds differ for two additional reasons:
\begin{itemize}[itemsep=0pt,topsep=0pt]
    \item In the upper bound we need to ``round'' within each partition to lower bound the gap, whereas for the lower bound we need to upper bound the gap.  A similar limitation was also present in \cite{aziz2018pure}, and as noted therein, the difference becomes less significant as $\epsilon$ decreases (or in their setting, as $\alpha$ decreases).
    \item The lower bound discussed above only characterizes the number of arm pulls on one side of the $(1-\alpha)$-quantile in each group, whereas the upper bound sums over all partitions (i.e., both sides of the quantile).  Similar limitations apply to the lower bounds in \cite{wang2021max}, though in our setting they are alleviated because our $\alpha$ value is fixed in $(0,1)$ (whereas \cite{wang2021max} roughly corresponds to taking $\alpha = 1$).
\end{itemize}
To summarize the above discussion, the overall gap terms that we introduced share many similar features and properties, but there may still remains significant room for bringing them closer together.

{\bf Limitations.}  
Unfortunately, even in cases where near-matching gap terms are attained in the upper and lower bounds, the lower bound approach discussed above would still fail to capture any counterpart to the $3\epsilon N$ factor present in \eqref{eq:armpullbound_independent_H} (which has dependence $\frac{1}{\epsilon}$), e.g., only attaining $\frac{1}{\Delta^2 \epsilon}$ scaling instead of $\frac{1}{\Delta^2 \epsilon^2}$ in \eqref{eq:weakened_final}.  We focused on worst-case lower bounds in Theorem \ref{thm:lb} because obtaining tight instance-dependent lower bounds appears to be significantly more challenging.

%% file: main_ARXIV.bbl
\begin{thebibliography}{39}
\providecommand{\natexlab}[1]{#1}
\providecommand{\url}[1]{\texttt{#1}}
\expandafter\ifx\csname urlstyle\endcsname\relax
  \providecommand{\doi}[1]{doi: #1}\else
  \providecommand{\doi}{doi: \begingroup \urlstyle{rm}\Url}\fi

\bibitem[Audibert et~al.(2010)Audibert, Bubeck, and
  Munos]{audibert2010bestarmid}
Jean-Yves Audibert, S{\'e}bastien Bubeck, and R{\'e}mi Munos.
\newblock Best arm identification in multi-armed bandits.
\newblock In \emph{Conference on Learning Theory}, pages 41--53, 2010.

\bibitem[Aziz et~al.(2018)Aziz, Anderton, Kaufmann, and Aslam]{aziz2018pure}
Maryam Aziz, Jesse Anderton, Emilie Kaufmann, and Javed Aslam.
\newblock Pure exploration in infinitely-armed bandit models with
  fixed-confidence.
\newblock In \emph{Conference on Algorithmic Learning Theory}, pages 3--24.
  PMLR, 2018.

\bibitem[Berry et~al.(1997)Berry, Chen, Zame, Heath, and
  Shepp]{berry1997bandit}
Donald~A Berry, Robert~W Chen, Alan Zame, David~C Heath, and Larry~A Shepp.
\newblock Bandit problems with infinitely many arms.
\newblock \emph{The Annals of Statistics}, 25\penalty0 (5):\penalty0
  2103--2116, 1997.

\bibitem[Bonald and Proutière(2013)]{bonald2013two}
Thomas Bonald and Alexandre Proutière.
\newblock Two-target algorithms for infinite-armed bandits with {B}ernoulli
  rewards.
\newblock In \emph{Conference on Neural Information Processing Systems}, 2013.

\bibitem[Bubeck et~al.(2008)Bubeck, Stoltz, Szepesv{\'a}ri, and
  Munos]{bubeck2008online}
S{\'e}bastien Bubeck, Gilles Stoltz, Csaba Szepesv{\'a}ri, and R{\'e}mi Munos.
\newblock Online optimization in {X}-armed bandits.
\newblock In \emph{Conference on Neural Information Processing Systems}, 2008.

\bibitem[Bubeck et~al.(2011)Bubeck, Munos, Stoltz, and
  Szepesv{\'a}ri]{bubeck2011x}
S{\'e}bastien Bubeck, R{\'e}mi Munos, Gilles Stoltz, and Csaba Szepesv{\'a}ri.
\newblock X-armed bandits.
\newblock \emph{Journal of Machine Learning Research}, 12\penalty0 (5), 2011.

\bibitem[Bubeck et~al.(2013)Bubeck, Wang, and
  Viswanathan]{bubeck2013multipleident}
Séebastian Bubeck, Tengyao Wang, and Nitin Viswanathan.
\newblock Multiple identifications in multi-armed bandits.
\newblock In \emph{International Conference on Machine Learning}, 2013.

\bibitem[Carpentier and Valko(2015)]{carpentier2015simple}
Alexandra Carpentier and Michal Valko.
\newblock Simple regret for infinitely many armed bandits.
\newblock In \emph{International Conference on Machine Learning}, 2015.

\bibitem[Chaudhuri and Kalyanakrishnan(2018)]{chaudhuri2018quantile}
Arghya~Roy Chaudhuri and Shivaram Kalyanakrishnan.
\newblock Quantile-regret minimisation in infinitely many-armed bandits.
\newblock In \emph{Conference on Uncertainty in Artificial Intelligence}, 2018.

\bibitem[Chaudhuri and Kalyanakrishnan(2019)]{chaudhuri2019pac}
Arghya~Roy Chaudhuri and Shivaram Kalyanakrishnan.
\newblock {PAC} identification of many good arms in stochastic multi-armed
  bandits.
\newblock In \emph{International Conference on Machine Learning}, 2019.

\bibitem[David and Shimkin(2014)]{david2014infinitely}
Yahel David and Nahum Shimkin.
\newblock Infinitely many-armed bandits with unknown value distribution.
\newblock In \emph{European Conference on Machine Learning and Knowledge
  Discovery in Databases}, 2014.

\bibitem[Gabillon et~al.(2012)Gabillon, Victor, Ghavamzadeh, Mohammad, and
  Lazaric]{gabillon2012unified}
Gabillon, Victor, Ghavamzadeh, Mohammad, and Alessandro Lazaric.
\newblock {Best arm identification: A unified approach to fixed budget and
  fixed confidence}.
\newblock In \emph{Conference on Neural Information Processing Systems}, 2012.

\bibitem[Gabillon et~al.(2011)Gabillon, Ghavamzadeh, Lazaric, and
  Bubeck]{gabillon2011MultiBandit}
Victor Gabillon, Mohammad Ghavamzadeh, Alessandro Lazaric, and S\'{e}bastien
  Bubeck.
\newblock Multi-bandit best arm identification.
\newblock In \emph{Conference on Neural Information Processing Systems}, 2011.

\bibitem[Garivier and Kaufmann(2016)]{garivier2016optimal}
Aur{\'e}lien Garivier and Emilie Kaufmann.
\newblock Optimal best arm identification with fixed confidence.
\newblock In \emph{Conference on Learning Theory}, pages 998--1027, 2016.

\bibitem[Grill et~al.(2015)Grill, Valko, Munos, and Munos]{grill2015black}
Jean-Bastien Grill, Michal Valko, Remi Munos, and Remi Munos.
\newblock Black-box optimization of noisy functions with unknown smoothness.
\newblock In \emph{Conference on Neural Information Processing Systems},
  volume~28, 2015.

\bibitem[Gupta et~al.(2020)Gupta, Chaudhari, Mukherjee, Joshi, and
  Ya{\u{g}}an]{gupta2020unified}
Samarth Gupta, Shreyas Chaudhari, Subhojyoti Mukherjee, Gauri Joshi, and Osman
  Ya{\u{g}}an.
\newblock A unified approach to translate classical bandit algorithms to the
  structured bandit setting.
\newblock \emph{IEEE Journal on Selected Areas in Information Theory},
  1\penalty0 (3):\penalty0 840--853, 2020.

\bibitem[Huang et~al.(2017)Huang, Ajallooeian, Szepesv{\'a}ri, and
  M{\"u}ller]{huang2017structured}
Ruitong Huang, Mohammad~M Ajallooeian, Csaba Szepesv{\'a}ri, and Martin
  M{\"u}ller.
\newblock Structured best arm identification with fixed confidence.
\newblock In \emph{Conference on Algorithmic Learning Theory}, 2017.

\bibitem[Jamieson(2022)]{JamiesonCSE599}
Kevin Jamieson.
\newblock {CSE599i: Online and Adaptive Machine Learning, Lecture 4}.
\newblock
  \url{https://courses.cs.washington.edu/courses/cse599i/18wi/resources/lecture4/lecture4.pdf},
  2022.
\newblock Accessed: 2022-09-15.

\bibitem[Jamieson and Nowak(2014)]{jamieson2014best}
Kevin Jamieson and Robert Nowak.
\newblock Best-arm identification algorithms for multi-armed bandits in the
  fixed confidence setting.
\newblock In \emph{Conference on Information Sciences and Systems}, pages 1--6,
  2014.

\bibitem[Jamieson et~al.(2014)Jamieson, Malloy, Nowak, and
  Bubeck]{jamieson2014lil}
Kevin Jamieson, Matthew Malloy, Robert Nowak, and S{\'e}bastien Bubeck.
\newblock lil’{UCB}: An optimal exploration algorithm for multi-armed
  bandits.
\newblock In \emph{Conference on Learning Theory}, pages 423--439. PMLR, 2014.

\bibitem[Jamieson et~al.(2016)Jamieson, Haas, and Recht]{jamieson2016power}
Kevin Jamieson, Daniel Haas, and Benjamin Recht.
\newblock The power of adaptivity in identifying statistical alternatives.
\newblock In \emph{Conference on Neural Information Processing Systems}, 2016.

\bibitem[Kalvit and Zeevi(2021)]{kalvit2021bandits}
Anand Kalvit and Assaf Zeevi.
\newblock Bandits with dynamic arm-acquisition costs.
\newblock \emph{arXiv preprint arXiv:2110.12118}, 2021.

\bibitem[Kalyanakrishnan et~al.(2012)Kalyanakrishnan, Tewari, Auer, and
  Stone]{kalyanakrishnan2012pac}
Shivaram Kalyanakrishnan, Ambuj Tewari, Peter Auer, and Peter Stone.
\newblock {PAC} subset selection in stochastic multi-armed bandits.
\newblock In \emph{International Conference on Machine Learning}, 2012.

\bibitem[Katz-Samuels and Jamieson(2020)]{katz2020true}
Julian Katz-Samuels and Kevin Jamieson.
\newblock The true sample complexity of identifying good arms.
\newblock In \emph{International Conference on Artificial Intelligence and
  Statistics}, 2020.

\bibitem[Kaufmann and Kalyanakrishnan(2013)]{pmlr-v30-Kaufmann13}
Emilie Kaufmann and Shivaram Kalyanakrishnan.
\newblock Information complexity in bandit subset selection.
\newblock In \emph{Conference on Learning Theory}, 2013.

\bibitem[Kaufmann et~al.(2016)Kaufmann, Capp{\'e}, and
  Garivier]{kaufmann2016complexity}
Emilie Kaufmann, Olivier Capp{\'e}, and Aur{\'e}lien Garivier.
\newblock On the complexity of best-arm identification in multi-armed bandit
  models.
\newblock \emph{Journal of Machine Learning Research}, 17\penalty0
  (1):\penalty0 1--42, 2016.

\bibitem[Kleinberg et~al.(2008)Kleinberg, Slivkins, and
  Upfal]{kleinberg2008multi}
Robert Kleinberg, Aleksandrs Slivkins, and Eli Upfal.
\newblock Multi-armed bandits in metric spaces.
\newblock In \emph{ACM Symposium on Theory of Computing}, 2008.

\bibitem[Kleinberg et~al.(2019)Kleinberg, Slivkins, and
  Upfal]{kleinberg2019bandits}
Robert Kleinberg, Aleksandrs Slivkins, and Eli Upfal.
\newblock Bandits and experts in metric spaces.
\newblock \emph{Journal of the ACM (JACM)}, 66\penalty0 (4):\penalty0 1--77,
  2019.

\bibitem[Lattimore and Szepesv\'ari(2020)]{Csa18}
Tor Lattimore and Csaba Szepesv\'ari.
\newblock \emph{Bandit Algorithms}.
\newblock Cambridge University Press, 2020.

\bibitem[Li and Xia(2017)]{li2017infinitely}
Haifang Li and Yingce Xia.
\newblock Infinitely many-armed bandits with budget constraints.
\newblock In \emph{AAAI Conference on Artificial Intelligence}, 2017.

\bibitem[Mukherjee et~al.(2020)Mukherjee, Tripathy, and
  Nowak]{mukherjee2020generalized}
Subhojyoti Mukherjee, Ardhendu Tripathy, and Robert Nowak.
\newblock Generalized {C}hernoff sampling for active learning and structured
  bandit algorithms.
\newblock \emph{arXiv preprint arXiv:2012.08073}, 2020.

\bibitem[Neopane et~al.(2021)Neopane, Ramdas, and Singh]{neopane2021best}
Ojash Neopane, Aaditya Ramdas, and Aarti Singh.
\newblock Best arm identification under additive transfer bandits.
\newblock In \emph{Asilomar Conference on Signals, Systems, and Computers},
  2021.

\bibitem[Ren et~al.(2019)Ren, Liu, and Shroff]{ren2019exploring}
Wenbo Ren, Jia Liu, and Ness~B Shroff.
\newblock Exploring $k$ out of top $\rho$ fraction of arms in stochastic
  bandits.
\newblock In \emph{Conference on Artificial Intelligence and Statistics}, 2019.

\bibitem[Scarlett et~al.(2019)Scarlett, Bogunovic, and
  Cevher]{scarlett2019overlapping}
Jonathan Scarlett, Ilija Bogunovic, and Volkan Cevher.
\newblock Overlapping multi-bandit best arm identification.
\newblock In \emph{IEEE International Symposium on Information Theory}, 2019.

\bibitem[Slivkins(2019)]{slivkins2019introduction}
Aleksandrs Slivkins.
\newblock Introduction to multi-armed bandits.
\newblock \emph{Foundations and Trends® in Machine Learning}, 12\penalty0
  (1-2):\penalty0 1--286, 2019.

\bibitem[Tan et~al.(2022)Tan, Jagannathan, et~al.]{tan2022survey}
Vincent Y~F Tan, Krishna Jagannathan, et~al.
\newblock A survey of risk-aware multi-armed bandits.
\newblock \emph{arXiv preprint arXiv:2205.05843}, 2022.

\bibitem[Wang et~al.(2008)Wang, Audibert, and Munos]{wang2008infinitely}
Yizao Wang, Jean-Yves Audibert, and R{\'e}mi Munos.
\newblock Algorithms for infinitely many-armed bandits.
\newblock In \emph{Conference on Neural Information Processing Systems}, 2008.

\bibitem[Wang and Scarlett(2021)]{wang2021max}
Zhenlin Wang and Jonathan Scarlett.
\newblock Max-min grouped bandits.
\newblock In \emph{AAAI Conference on Artificial Intelligence}, 2021.

\bibitem[Zhang and Ong(2021)]{zhang2021quantile}
Mengyan Zhang and Cheng~Soon Ong.
\newblock Quantile bandits for best arms identification.
\newblock In \emph{International Conference on Machine Learning}, 2021.

\end{thebibliography}
